\relax
\documentclass[letterpaper]{article} 
\usepackage{times}  
\usepackage{helvet}  
\usepackage{courier}  
\usepackage[hyphens]{url}  
\usepackage{graphicx} 
\usepackage{caption} 
\DeclareCaptionStyle{ruled}{labelfont=normalfont,labelsep=colon,strut=off} 
\usepackage{algorithm}
\usepackage[algo2e,ruled,linesnumbered,noend]{algorithm2e}
\usepackage[noend]{algorithmic}
\usepackage{amsmath}
\usepackage{amsfonts}
\usepackage{bbm}
\usepackage{floatflt}
\usepackage{multirow}
\usepackage{enumitem}
\usepackage{tikz}
\usetikzlibrary{positioning}
\usepackage{booktabs}
\usepackage{fullpage}
\usepackage{hyperref}

\newcommand{\qed}{\hfill \ensuremath{\Box}}

\newtheorem{thmx}{Theorem}

\newcommand{\pluseq}{\mathrel{+}=}
\newcommand{\minuseq}{\mathrel{-}=}

\newenvironment{proof}{{\bf Proof:}}{\qed}

\SetKw{Kw}{function}
\SetKwComment{Comment}{$\triangleright$ }{ }
\SetKwInput{KwInput}{Input}
\SetKwInput{KwOutput}{Output}

\tikzset{main node/.style={circle,fill=blue!20,draw,minimum size=1cm,inner sep=0pt},
            }

\usepackage{newfloat}
\usepackage{listings}
\floatstyle{ruled}
\newfloat{listing}{tb}{lst}{}
\floatname{listing}{Listing}
\title{Learning Graph Node Embeddings by Smooth Pair Sampling}
\author{
Konstantin Kutzkov \\ \normalsize kutzkov@gmail.com
}
\date{}
\begin{document}

\maketitle

\begin{abstract}
Random walk-based node embedding algorithms have attracted a lot of attention due to their scalability and ease of implementation. Previous research has focused on different walk strategies, optimization objectives, and embedding learning models.
Inspired by observations on real data, we take a different approach and propose a new regularization technique. More precisely, the frequencies of node pairs generated by the skip-gram model on random walk node sequences follow a highly skewed distribution which causes learning to be dominated by a fraction of the pairs. We address the issue by designing an efficient sampling procedure that generates node pairs according to their {\em smoothed frequency}. Theoretical and experimental results demonstrate the advantages of our approach.
\end{abstract}

\section{Introduction}
\label{intro}

Representation learning from graphs has been an active research area over the past decade. DeepWalk~\cite{deepwalk}, one of the pioneering approaches in this field, learns node embeddings by generating random walks on the graph. A standard and highly efficient method for optimizing the embedding objective is to train a binary classification model that distinguishes between positive and negative node pairs, known as the negative sampling approach. Positive pairs are generated by applying the skip-gram model~\cite{word2vec} to the node sequences and represent nodes whose embeddings should be similar, in contrast to negative pairs.

Many works have since extended the original DeepWalk algorithm in three main directions: i) presenting different (random) walk strategies for traversing the graph and generating node sequences~\cite{el_haija,metapath2vec,node2vec,walklets,diff2vec,line,attri2vec,app}, ii) designing new embedding learning models~\cite{el_haija,watch_your_step,graphsage,covar_loss,matfac}, and iii) designing new techniques for negative pair sampling~\cite{robust_neg_sampling,res2vec,adaptive_neg_sampling,neg_smooth}. Inspired by observations on real graphs, we take a different approach and propose a general regularization technique that adjusts the frequency distribution of positive node pairs.

In the standard negative sampling setting, when a positive pair $u, v$ is generated, we also sample $k \ge 1$ negative pairs $u, x$, where the node $x$ is selected at random from some distribution on the graph nodes. Usually, this is the $\alpha$-smoothed node degree distribution $\frac{d(u)^\alpha}{\sum_{u \in V} d(u)^\alpha}$, where $d(u)$ is the degree of node $u$ and $\alpha \in (0, 1]$ is a hyperparameter. The pairs are then provided as input to a binary classification model that learns the node embeddings.

We propose to apply frequency smoothing to positive pairs. If a pair $u, v$ occurs $\#(u,v)$ times in the random walk corpus $D$, after smoothing it will be provided as a positive example to the classification model approximately $T_\beta \#(u, v)^\beta$ times, where $T_\beta$ increases with decreasing $\beta \in (0, 1]$ and $\sum_{u,v \in D}\#(u,v)=\sum_{u,v \in D}T_\beta \#(u,v)^\beta$.
%
\begin{figure}[t]
\centering
\includegraphics[scale=0.4]{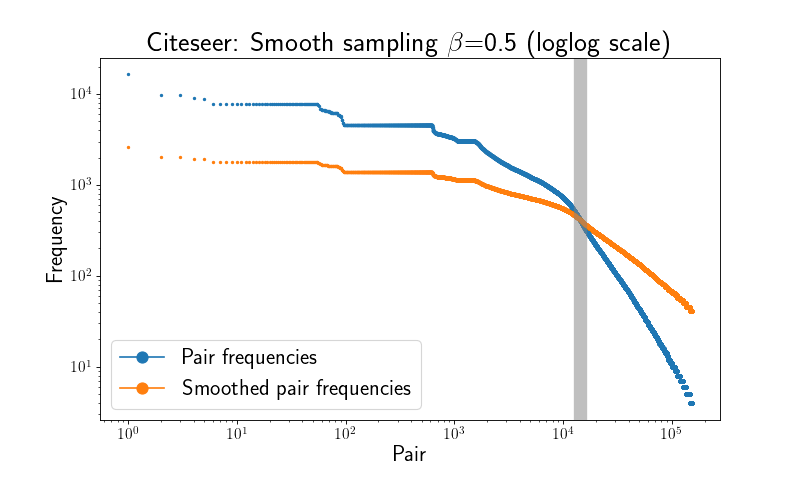}
\caption{The effect of smoothing the pair frequencies for the Citeseer graph. The gray vertical line shows the transition point after which smoothing leads to more positive samples for the corresponding pairs.}
\label{fig:smoothing_effect}
\end{figure}
\\\\
{\bf Motivation and paper contribution.}
Before introducing yet another hyperparameter like the smoothing exponent $\beta$ we must consider the following questions:
\begin{itemize}[leftmargin=*]     
\item {\em Are there any insights suggesting that frequency smoothing might be beneficial?}  We present theoretical results in Section~\ref{sec:why}, but at a high level, we argue that smoothing enhances the robustness of the learning process. Random walk-based embeddings inherently assume that the most frequent positive pairs are the most important, and that embeddings should preserve their similarity. However, in real graphs, the frequency distribution is highly skewed. Figure~\ref{fig:smoothing_effect} illustrates how the frequency distribution of positive pairs generated by DeepWalk changes after applying smoothing with $\beta = 0.5$ for the Citeseer graph. We observe that smoothing decreases the frequency for fewer than 5\% of pairs (note the log scale). The smoothing hyperparameter $\beta \in (0,1]$ acts as a regularizer, applying a “progressive tax” on pair frequencies. While the “rich” pairs remain dominant, we allow the emergence of “middle-class” pairs. Consequently, smoothing extends the set of node pairs whose similarities are preserved by the embeddings. 
\item {\em Can frequency smoothing be efficiently implemented?} Applying the approach used for negative sampling would be to compute and explicitly store the frequencies of all pairs that appear in the random walk sequences and then sample according to the smoothed frequencies. But as we show in Section~\ref{sec:experiments}, this would be unfeasible for larger graphs because the number of pairs grows superlinearly with the graph size. 
We present a simple and highly efficient technique for smooth pair sampling which can be of independent interest. Our algorithm leverages data sketching techniques which provably ensures its scalability. 
\end{itemize}
{\bf Organization of the paper}
In the next section we define the problem setting.  In Section~\ref{sec:algo} we present SmoothDeepWalk, an algorithm that achieves smooth positive pair sampling by only slightly modifying DeepWalk. 
We then provide a theoretical analysis of the benefits of smoothing for node embeddings in Section~\ref{sec:why} and, to address scalability, present an efficient data sketching approach for large graphs in Section~\ref{sec:rtime}.
Related work is discussed in Section~\ref{sec:related}. Experimental evaluation is presented in Section~\ref{sec:experiments} and the paper is concluded in Section~\ref{sec:concl}.

\section{Preliminaries} \label{sec:prel}

We consider undirected connected graphs $G=(V,E)$ but all presented algorithms also work for directed graphs. Let $n=|V|$. The {\em degree} of a node $u \in V$ is the number of edges $(u, v) \in E$ and is denoted by $d(u)$. The set of {\em neighbors} of node $u$ is $N(u) = \{v \in V: (u, v) \in E\}$. A {\em walk} $w$ of length $\ell$ on $G$ is a sequence of $\ell$ nodes in $G$ such that $(w_i, w_{i+1}) \in E$ for $i \in \{1,2,\ldots, \ell-1\}$. 
Given a set of walks $W$, we define the {\em pair corpus} to be the multiset of {\em positive node pairs} $D = \{(w_i, w_j): \exists w \in W$ such that  $|i-j| \le t, i\neq j\}$ for a user-defined window size $t\ge 1$.

The {\em total} number of pairs is $M = |D|$, and we denote the number of {\em unique} pairs in $D$ by $P$. Further, $M_\beta = \sum_{u, v \in D} \#(u,v)^\beta$ for $\beta \ge 0$, where $\#(u,v)$ is the frequency of $u, v$ in $D$. We set $\#u^{(\beta)} = \sum_{v \in V} \#(u, v)^\beta$. Let the pairs be sorted in descending order by their frequency and denote the $i$-th pair frequency in $D$ as $f_i(D)$. We say that the $i$-th pair has {\em rank} $i$. It holds $M=\sum_{u, v \in D} \#(u,v)=\sum_{i=1}^P f_i(D)$.

The $d$-dimensional {\em vector embedding} of node $u$ is denoted by $\vec{u} \in \mathbb{R}^d$.

At a high level, we propose to replace the original corpus $D$ by a {\em $\beta$-smoothed} corpus $D_\beta$ in which the original frequencies $\#(u,v)$ are smoothed to $T_\beta \#(u,v)^\beta$ with $T_\beta = \lceil M/M_\beta \rceil$.  Observe that $T_\beta \in [1, M/P)$ for $\beta \in (0,1]$ and monotonically increases with decreasing~$\beta$. Also, the cardinality of $D_\beta$ remains $M$. 
\\
{\bf Frequency distribution.}
A common statistical formalization of the skewness in real-life datasets is the assumption of Zipfian distribution with parameter $z \ge 0$. Let $S_z = \sum_{i=1}^P i^{-z}$ and $X_z = M/S_z$. In our setting, we assume the pair frequency for the $i$-th pair is $f_i = \frac{X_z}{i^z}$ and it thus holds $\sum_{i=1}^P \frac{X_z}{i^z} = M$.

{\bf Training objective.}
The embedding algorithm maximizes 
\[J = \sum_{(u,v) \in D}\log \frac{\exp(\vec{u}^T \vec{v})}{\sum_{x \in V} \exp(\vec{u}^T \vec{x})}\,\]

Let $\mu: V \rightarrow (0, 1]$ be a distribution on the graph nodes. The above objective is efficiently approximated by the following alternative objective which uses {\em negative sampling}, i.e., for a positive pair $u,v \in D$ we generate $k\ge 1$ negative pairs $u, x$ from $\mu$: 
\[J = \sum_{(u, v) \in D} \log \sigma(\vec{u}^T \vec{v}) + k \cdot \mathbb{E}_{x\sim \mu} \log \sigma(-\vec{u}^T \vec{x})\,\] where $\sigma(x) = (1+\exp(-x))^{-1}$ is the sigmoid function.
A popular choice for $\mu$ is the smoothed degree distribution $\mu(v) = \frac{d(v)^\alpha}{\sum_{v \in V}d(v)^\alpha}$ for $\alpha \in [0, 1]$.

\section{Smooth Pair Sampling} \label{sec:algo}

\SetInd{0.75em}{0.25em}
\begin{algorithm2e}
\small
\caption{SmoothDeepWalk}\label{alg:main}
\KwInput{graph $G=(V,E)$ on $n$ nodes, random walks per node $s$, walk length $\ell$, window size $t$, embedding dimension $d$, smoothing exponent $\beta \in (0,1]$, number of negative samples $k$}
\KwOutput{Embedding matrix $\Phi \in \mathbb{R}^{n \times d}$}
$pos\_pairs \gets 0$\\
$M = n \cdot s \cdot $ \texttt{number\_pairs\_per\_sequence}($\ell$, $t$) \hspace*{3mm}\Comment{\scriptsize $M$ is the total number of positive pairs}  
\While{$pos\_pairs < M$}{
 {Iterate over nodes $w \in V$ and for each $w$ generate $s$ random walks $T$ of length $\ell$ starting at $w$}\\
 \For{$i = 1 \text{ to } \ell$ \hspace*{4mm}\Comment{\scriptsize skip-gram pair generation}} { 
 	 $u = T[i]$\\
	 $start = \max(0, i-t)$\\
	 $end = \min(\ell, i+t)$\\
	 \For{$j \in [start:i, i+1:end]$}{
	 	$v = T[j]$\\
        Generate $r \in \mathbb{U}[0,1)$ \hspace*{3mm}\Comment{\scriptsize smooth sampling} 
        \If{$r \le \#(u, v)^{\beta-1}$}{
			 $pos\_pairs =pos\_pairs + 1$\\
			 Sample $k$ negative pairs $(v,v_N)$ \\
			 Feed the $k+1$ pairs and labels into $\mathcal{M}(\Phi)$
	 }
 	}
  }
}
\Return $\Phi$
\end{algorithm2e}

We present an algorithm that builds upon the positive pair generation approach used in the original DeepWalk algorithm and smooths the pair frequencies on the fly.
Algorithm~\ref{alg:main} outlines how the proposed SmoothDeepWalk works. The total number of positive pairs $M$ is a function of the number of nodes, the number of walks per node, the walk length and the window size.  We iterate over the graph nodes, from each node we start $s$ random walks, each of length $\ell$, and from each node sequence generate positive node pairs using the skip-gram approach~\cite{word2vec}. For each positive pair we sample $k\ge 1$ negative node pairs that are fed into a binary classification model with an embedding matrix $\Phi$, until we have sampled $M$ positive pairs. The only difference to DeepWalks is in lines 13-14 where we sample a candidate pair $u,v$ with probability $\#(u, v)^{\beta-1}$ for $\beta \in (0, 1]$. Note that for $\beta=1$ we have the standard DeepWalk algorithm. 
In a single pass over the random walk corpus we observe the pair $u,v$ exactly $\#(u, v)$ times, thus we expect to sample it  $\#(u, v)^{\beta}$ times. We also expect $T_\beta = M/M_\beta$ passes. Using that individual samples are independent, we show the following result (proof in Appendix~\ref{sec:app_proofs}): 

\begin{thmx} \label{thm:sampling}
Let $G=(V, E)$ and $D$ be the corresponding random walk corpus. Let $\beta \in (0, 1]$.  The following hold for SmoothDeepWalk:
\begin{itemize}[leftmargin=*]
\item Let $M_\beta \ge c T_\beta(T_\beta + 1)$ for some $c\ge1$.  With probability at least $1-e^{-c}$, SmoothDeepWalk needs between $T_\beta -1$ and $T_\beta +1$ passes over the corpus $D$.
\item Let $S_{u,v}$ be the number of positive samples of pair $u, v$ returned by SmoothDeepWalk. It holds $\mathbb{E}(S_{u,v}) = T_\beta\#(u, v)^\beta$. If $\#(u,v) \ge 1/\varepsilon^2 \log 1/\delta$ for $\varepsilon, \delta \in (0, 1)$, then with probability $1-\delta$ it holds $|S_{u,v} - T_\beta(\#(u,v)^{\beta}| \le \varepsilon \#(u, v)$.
\end{itemize}
\end{thmx}

The theorem shows that SmoothDeepWalk accurately smooths the pair frequencies and despite the random sampling step, the running time of the algorithm is predictable and almost certain. 
For any reasonable choice of the hyperparameters $\ell, t, s$ and $\beta$, $T_\beta = M/M_\beta$ is a constant and the condition $M_\beta \ge c T_\beta(T_\beta + 1)$ is satisfied for a large value of $c$.
%

\section{How does smoothing help?} \label{sec:why}
As already discussed, for a smoothing parameter $\beta \in (0, 1]$ the pair frequencies are transformed to $T_\beta \#(u,v)^\beta$. The cardinality of the most frequent pairs thus decreases, and of less frequent pairs increases. The vertical gray line in Figure~\ref{fig:smoothing_effect} shows the transition point.  Under the assumption that pair frequencies follow Zipfian distribution with parameter $z$, we first analyze what is the pair rank when this transition occurs depending on $z$ and $\beta$, i.e., how the data skew and the smoothing level affect the location of the gray line.

\begin{thmx} \label{thm:smooth_transition}
Let $D$ be a corpus of cardinality $M$ of node pairs. Let the frequencies in $D$ follow a Zipfian distribution with parameter $z \ge 0$, and $D_\beta$  be the $\beta$-smoothed corpus for $\beta \in (0, 1]$. Let $j$ be the minimum pair rank such that $f({D_\beta})_j > f({D})_j$. 
\begin{itemize}[leftmargin=*]
\item If $z>1$ and $\beta z > 1$, then $j = c(z, \beta)$ for some constant $c$. 
\item If $z>1$ and $\beta z < 1$, then $j = P^{\frac{1-\beta z}{z -\beta z}} = o(P)$. 
\item If $z<1$, then $j = O(P)$. 
\end{itemize}
\end{thmx}

The above result has a very intuitive interpretation.  For highly skewed distributions, a slight decrease in the skew such that $\beta z>1$ would reduce the frequency only for a few of the most frequent pairs. If the decrease is more significant, i.e., $\beta z < 1$ (one can visually interpret this as the ``angle'' between the blue and orange lines in Figure~\ref{fig:smoothing_effect} becoming larger), then we move $j$, i.e., the gray vertical line, to the right. Still, the frequency only for a sublinear number of the pairs decreases. And if the distribution is not very skewed, then we will increase the frequency only for a constant fraction of the pairs.  

We next analyze how smoothing impacts the embeddings of individual nodes. The proof of the next theorem follows from~\cite{word2vec_expl} and can be found in  Appendix~\ref{sec:app_proofs}. 

\begin{thmx} \label{thm:matfac}
Let $D_\beta$ be $\beta$-smoothed corpus of cardinality $M_\beta$ for $\beta \in (0, 1]$. Let $\mu: V \rightarrow [0, 1]$ be the negative sampling probability distribution.  For all node pairs $u, v \in V$, SmoothDeepWalk optimizes the objective $\vec{u}^T \vec{v} = \log \frac{\#(u, v)^\beta}{\#u^{(\beta)} \mu(v)} - \log k$.
\end{thmx}

Recall that $\#u^{(\beta)} = \sum_{w \in V} \#(u, w)^\beta$. The negative sampling distribution $\mu(v)$ is independent from $u$. We analyze the expression $\frac{\#u^{(\beta)}}{\#(u, v)^\beta}$ under the assumption that for the $n$ pairs $u, w$ with $w \in V$, the frequencies follow a Zipfian distribution with parameter $z>1$. Let the rank of the pair $u, v$ be $i$, and its frequency be $O(1/i^z)$. We have $\frac{\#u^{(\beta)}}{\#(u, v)^\beta} = \sum_{j=1}^n (\frac{i}{j})^{\beta z}$. We distinguish the following cases: 
\begin{itemize}[leftmargin=*]
\item[-] 
If $i$ is one of the top ranks among the $u, w$ pairs, then we have $\frac{\#u}{\#(u, v)} = O(1)$ for $z > 1$.  The objective then optimizes $O(\mu(v)^{-1})$. By smoothing, we can bound $\frac{\#u^{(\beta)}}{\#(u, v)^\beta} = O(n^{1-\beta z})$ for $\beta z < 1$ which regularizes the optimization objective for high-frequency pairs from $O(1/\mu(v))$ to $O(\frac{n^{\beta z -1}}{\mu(v)})$.
\item[-] 
For $i > k$, we have  $\frac{\#u}{\#(u, v)} > \sum_{j=1}^n (k/j)^z = O(k^z)$. In particular, for $\mu(v) > 1/k^z$ the number of negative samples of $u, v$ will exceed the number of positive samples. (Note that this might explain the observation that hierarchical softmax in word2vec is better for infrequent words than negative sampling~\cite{word2vec}.) 
By smoothing the frequencies such that $\beta z < 1$, we get $ \sum_{j=1}^n (k/j)^{\beta z} = O(k^{\beta z} n^{1-\beta z})$. Thus, if the rank of $u, v$ is $k > n^{\frac{1-\beta z}{z - \beta z}} = o(n)$ it holds $\frac{\#(u, v)^\beta}{\#u^{(\beta)}} > \frac{\#(u, v)}{\#u}$. Thus, the representation for low frequency pairs improves. 
\end{itemize}

In summary, Theorem~\ref{thm:smooth_transition} shows how smoothing affects the overall distribution of positive samples, and Theorem~\ref{thm:matfac} presents a generalized version of DeepWalk's optimization objective that distributes the contribution of individual pairs more evenly. This highlights the motivation behind smoothing as a tool that allows to extend the set of node pairs that are considered important.

\section{Efficient Pair Frequency Estimation} \label{sec:rtime}
The obvious drawback of Algorithm 1 is that we need to know the frequencies $\#(u,v)$ of all $P$ pairs in the corpus, and thus it needs $O(P)$ memory. It holds $P=O(n\cdot \ell \cdot t \cdot s)$ and when  using the default values for the hyperparameters $\ell, t, s$, for the 8 real graphs considered in our experimental evaluation it holds $P=O(n^c)$ for $c > 1.64$. Thus, for larger graphs the amount of available memory on commodity machines will likely be insufficient. Solutions that use external memory like Glove~\cite{glove} can be impractical for various reasons.  
We address the issue by leveraging algorithms for frequent item mining in data streams, a widely studied data mining problem~\cite{hh_mining}. We use a compact summary of the pair frequency distribution that allows us to approximate the frequency of each node pair $(u, v)$. 

We describe how the {\sc Frequent} method~\cite{frequent} detects heavy hitters in two passes over a data stream, pseudocode can be found in Appendix~\ref{sec:app_code}.
{\sc Frequent} maintains a dictionary $S$, the so-called {\em sketch}, that stores up to $b$ distinct pairs together with a counter lower bounding the frequency of each pair.
When a new pair $(u,v)$ is generated, we check whether it is already in $S$. If so, we increment the corresponding counter. Otherwise, we insert $(u,v)$ with a counter set to 1.  If there are $b+1$ pairs in $S$, we decrease by 1 the counter of all those pairs and remove pairs whose counter has become 0, so we guarantee at most $b$ pairs remain in $S$. The last step corresponds to removing $b+1$ distinct pairs from the multiset of positive pairs. The total number of pairs is $M$ and the weight of a pair can be underestimated by at most $M/b$. Thus, all pairs that appear more than $M/b$ times are guaranteed to be in $S$ after processing all pairs.   
We implement $S$ as a hash table and charge the cost of each counter decrement to the cost incurred at its arrival. Thus, the amortized cost per update is constant, and the total running time is $O(M)$. 

In a second pass over the corpus (we assume that the random walks are created by setting a random seed and the corpus is generated on the fly), we compute the exact frequency of all pairs recorded in the sketch. In this way, in SmoothDeepWalk we work with the exact frequencies of the heavy pairs. For pairs not recorded in the sketch, we return an overestimation of the pair frequency as  
$(M - \textsl{total weight of pairs in sketch})/b$. (In Appendix~\ref{sec:app_extensions} we discuss alternative sketching methods.)

The following result shows that for a skewed frequency distribution we need a compact sketch in order to provably detect the most frequent pairs. The proof is based on a result by~\cite{berinde_et_al} and is provided in  Appendix~\ref{sec:app_proofs}.  
\begin{thmx} \label{thm:sketchsize}
Let the pair frequencies follow Zipfian distribution with parameter $z$. For $z>1$, {\sc Frequent} returns the exact frequency of the $k$ most frequent pairs using a sketch of size $b=O(k)$. For $z<1$, the required sketch size is $b=O(k^zP^{1-z})$ where $P$ is the number of unique pairs.
\end{thmx}

{\bf Computational complexity.}
The space complexity of Algorithm~1 is $O(n\cdot d + b)$ where $b$ is the number of pairs in the sketch and can be chosen according to the amount of available memory. The time complexity depends on the sampling probability $\#(u, v)^{\beta-1}$. In a single pass over the random walk sequences we expect to sample $M_\beta = \sum_{(u, v) \in \mathcal{G}}\#(u, v)^\beta$ pairs. Thus, the  expected number of passes is $T_\beta = \lceil M/M_\beta \rceil$.
However, the number of pairs  provided as input to the embedding learning algorithm is independent of $\beta$ and training the embeddings is the computationally intensive part. 

\section{Related work} \label{sec:related}

We give an overview of the main node embeddings approaches. The reader is referred to the survey by~\cite{survey_embs} for more details and a deeper discussion. 

Factorization of the graph adjacency matrix is an intuitive and efficient approach as most graphs are sparse and there exist efficient approaches to sparse matrix factorization~\cite{hope,arope}. 
Inspired by word2vec~\cite{word2vec}, DeepWalk~\cite{deepwalk} was the first algorithm that employs random walks and which in turn inspired other  approaches to exploit different (random) walk strategies~\cite{el_haija,metapath2vec,node2vec,walklets,diff2vec,line,attri2vec,app}. DeepWalk uses hierarchical softmax~\cite{word2vec}  for optimizing the embedding objective. Due to its efficiency,  negative sampling has become the standard optimization method and many negative sampling variants have been proposed~\cite{robust_neg_sampling,res2vec,adaptive_neg_sampling,neg_smooth}.
Building upon~\cite{word2vec_expl}, ~\cite{matfac} show that several of the above algorithms can be unified in a common matrix factorization framework and provide a precise definition of the corresponding matrix. A main advantage of word2veec and DeepWalk is that the matrix is never explicitly generated. 


As an alternative, discrete node embeddings represent nodes by vectors consisting of discrete features such as node attributes~\cite{lonesampler,nethash,nodesketch}. Discrete embeddings are compared using their Hamming distance and are interpretable if the entries are human readable features. The embedding models are efficient as there is no need to train a classification model. But they achieve lower accuracy and limit the choice of machine learning algorithms they can be combined with for downstream tasks.

Graph neural network models~\cite{gnn_deep,graphsage,gae} often yield state-of-the-art embedding methods. The models are inductive and can compute embeddings for previously unseen nodes by aggregating the embeddings of their neighbors. In contrast, the above discussed approaches are transductive and need to retrain the model to learn embeddings for new nodes. Also, GNNs naturally incorporate into the learning process node attributes and can be trained in end-to-end classification tasks. This comes at the price of the typical deep learning disadvantages such as scalability, tedious  hyperparameter tuning and the difficulty to obtain rigorous theoretical results like the ones in~\cite{matfac}.

{\bf Smooth negative sampling.}
Probably most relevant to our approach is the work by~\cite{neg_smooth}. 
%
The authors argue for smooth sampling of negative pairs from the positive distribution in order to guarantee that we don't oversample or undersample some negative pairs. Using only limited training data, it is shown that learning needs to trade off the objective of positive sampling and the expected risk that the embeddings deviate from the objective due to a limited number of samples. It is concluded that ``negative sampling distribution should be positively but sublinearly correlated to their positive sampling distribution''.  So instead of generating $\#(u,v)^\beta$ positive samples, we generate $\#(u,v)$ positive and $\#(u,v)^{1-\beta}$ negative samples. 
The big advantage of our approach compared to smooth negative sampling is its efficiency. In order to sample a negative node pair, \cite{neg_smooth} uses a complex MCMC sampling algorithm that requires the evaluation of the inner product for several candidate node pairs which needs time $\Omega(d)$. In contrast, SmoothDeepWalk needs a single hash table lookup per candidate pair and time $O(1)$.

\section{Experimental evaluation} \label{sec:experiments}

In this section, we introduce the graph datasets used for evaluation and discuss their properties. We then present the experimental setup and analyze how the performance of SmoothDeepWalk in link prediction is influenced by the sketch budget $b$ and the smoothing exponent $\beta$. We derive default values for $b$ and $\beta$ depending on the computational complexity and the graph structure. Using these default values, we evaluate the performance of smoothing for node classification, and present a smooth version of node2vec.  The performance of the proposed algorithms is then compared to other node embedding methods.  
%
Note that since the evaluation is extensive, we only present our main findings. We refer to Appendix~\ref{sec:app_experiments} for details and additional results. A prototype Python implementation~\footnote{\url{https://github.com/konstantinkutzkov/smooth_pair_sampling}} is publicly available, see Appendix~\ref{sec:app_data} for details about the implementation.
\paragraph{Datasets.}~\label{sec:data} 
%
We perform experiments on eight publicly available graphs summarized in Table~\ref{tab:datainfo}. The first three datasets Cora, Citeseer and Pubmed~\cite{sen_et_al} are citation networks, Git~\cite{musae}, Deezer~\cite{musae} and BlogCatalog~\cite{blogcatalog} represent different social interactions, Flickr~\cite{flickr} is a network of images connected by certain properties, and the Wikipedia graph~\cite{wikipedia} represents word co-occurrences. The graph nodes are labeled with one of \texttt{\#classes} unique labels, and Wikipedia and BlogCatalog can have more than one label per node.  Following the discussion in Section~\ref{sec:algo}, the column \texttt{\#pairs}  shows the (estimated) number of unique pairs $P$ as a function of the number of nodes $n$. 
The last column shows the graph clustering coefficient $CC$.

\begin{table}[!h]
\caption{Information on datasets.} \label{tab:datainfo}
\hspace*{-1cm}
\scriptsize
\centering
\begin{tabular}{ c  ccccc}
& nodes & edges & classes & \#pairs & $CC$\\
\toprule
Cora & 2.7K & 5.4K & 7  & $n^{1.703}$ & 0.241\\
\midrule
Citeseer & 3.3K & 4.7K & 28 & $n^{1.657}$ & 0.141\\
\midrule
PubMed & 19.7K & 44.3K &  18  & $n^{1.649}$ & 0.06\\
\midrule
Deezer & 28.3K & 92.8K & 2 & $n^{1.772}$ & 0.141 \\
\midrule
Git & 37.7K & 289K & 2 & $n^{1.733}$ & 0.168\\
\midrule 
Flickr & 89.3K & 450K & 7 &  $n^{1.687}$ & 0.033\\
\midrule 
Wikipedia & 4.8K & 92.5K & 40 & $n^{1.927}$ & 0.539\\
\midrule
BlogCatalog & 10.3K & 334K & 39 & $n^{1.861}$ & 0.463\\
\bottomrule 
\end{tabular}
\end{table}

\paragraph{Data skew.} The main motivation of the paper is the assumption that the frequency distribution for positive node pairs is highly skewed. In addition to Figure~\ref{fig:smoothing_effect}, in Appendix~\ref{sec:app_data} we present statistics for the other graphs that confirm the assumption is justified. Also, following the discussion in Section~\ref{sec:why}, we show that smoothing increases the ratio of positive to negative samples for middle-rank pairs. 

\paragraph{Experimental setting.} 
We implement pair smoothing in Python and conduct experiments on a standard laptop. We use a TensorFlow generator for smooth sampling and a shallow neural network for learning the embeddings, with the hyperparameters from~\cite{node2vec}: embedding dimension $d=128$, random walks of length 80, 10 walks per node, a window of size 10, 5 negative samples per positive pair, and negative pair sampling using the smooth node degree distribution with $\alpha=0.75$ discussed in Section~\ref{sec:prel}. Additional details can be found in Appendix~\ref{sec:app_data}.


\paragraph{Computational complexity.}  
Smoothing requires a sketch of the pair frequencies and several passes over the random walk corpus. In Figure~\ref{fig:varying_beta} we show how the training time increases for decreasing $\beta$ and fixed sketch budget $b=10\%$ of all $P$ unique pairs in the corpus. The time also increases by decreasing $b$ as we overestimate the frequency of more pairs and thus decrease the sampling probability. Overall, we observe that for larger graphs the increase in running time is rather small, we refer to Appendix~\ref{sec:app_hyperparameters} for a discussion. It should be noted that our algorithm is orders of magnitude faster than approaches to smooth negative pair sampling~\cite{neg_smooth,adaptive_neg_sampling}.
\begin{figure}[h!]
\centering
\includegraphics[width=60mm, height=40mm]{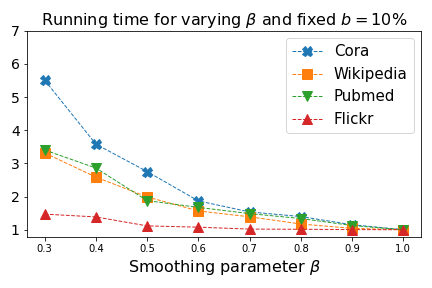}
\caption{Embedding training time (left)}
  \label{fig:train_time}
\end{figure}

\paragraph{Link prediction.} We focus on link prediction in the following analysis because it allows a more clear interpretation of the properties of the learned embeddings. Following the setting proposed in~\cite{hope,nodesketch}, we delete 20\% of the edges at random such that the graph remains connected and train embeddings on the reduced graph. 
We train embeddings for combinations of $\beta \in [0.3, \ldots, 0.9]$ and $b \in [1\%, 5\%, 10\%, 20\%]$ of the estimated value of $P$. For a random sample of node pairs, we select the $k$ pairs with the highest inner products, compute precision@$k$ and recall@$k$, and report the averages over 100 independent runs.

For detailed numerical scores and in-depth analysis, please refer to Appendix~\ref{sec:app_hyperparameters}. In summary, the results for link prediction are as follows:
\begin{itemize}
\item No statistically significant improvements for the small Cora and Citeseer graphs according to a $t$-test at significance level $0.01$.
\item Impressive improvements for the larger, sparse Pubmed, Deezer, Git, and Flickr graphs. 
\item Smaller yet statistically significant improvements for the denser Wikipedia and BlogCatalog.
\item The improvements are consistent for varying values of $k$ for both precision and recall.
\end{itemize} 

\paragraph{Default values for $\beta$ and $b$.} It might appear that one would need careful hyperparameter tuning for the optimal value of $\beta$ but fortunately the scores exhibit a functional dependency on $\beta$ and there is no need for exhaustive search. However, different graphs show different patterns. We show two representative examples in Figure~\ref{fig:varying_beta}. For the Git graph, the lower $\beta$ the stronger the result. For the Wikipedia graph too aggressive smoothing can be harmful. The Wikipedia graph exhibits a clear community structure as shown by its large clustering coefficient. We conjecture that too aggressive smoothing can destroy local communities. In order to test the hypothesis, we generate a scale-free graph with a low clustering coefficient, and then add edges at random among the neighbors of high-degree nodes. As we see in Figure~\ref{fig:synthetic}, smoothing yields improvements for lower values of $\beta$ only for the graph with a low clustering coefficient. 

\begin{figure}[t!]
\centering
\includegraphics[width=60mm, height=40mm]{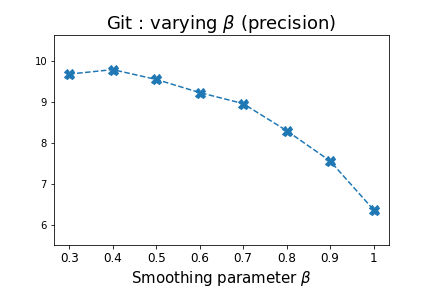}
\includegraphics[width=60mm, height=40mm]{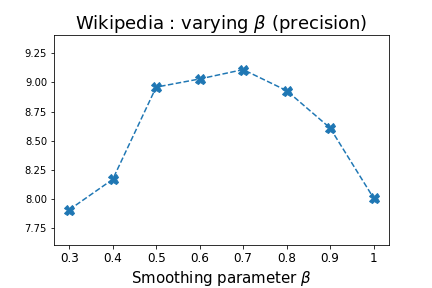}
  \caption{Precision@100 scores for link prediction for varying $\beta$ for the Git and Wikipedia graphs.}
  \label{fig:varying_beta}
\end{figure}

\begin{figure}[h!]
\centering
\includegraphics[width=60mm, height=40mm]{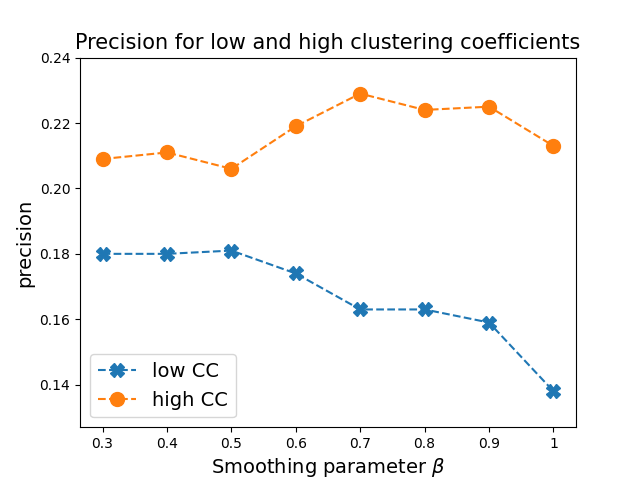}
  \caption{Precision@100 scores for synthetic graphs with low and high clustering coefficients.}
  \label{fig:synthetic}
\end{figure}

We observe the sketch budget $b$ has little influence on the accuracy, see Appendix~\ref{sec:app_hyperparameters} for details,  hence we set $b=10\%$ as it offers a good time-memory trade-off. We analyze all eight graphs in Appendix~\ref{sec:app_hyperparameters} and take into account the increase in running time for lower values of $\beta$ as well as the theoretical results from Section~\ref{sec:why} showing that too aggressive smoothing might lead to embeddings that preserve the similarity of almost all positive pairs. Ultimately, we recommend $\beta \in [0.4, 0.6]$ for graphs with a low clustering coefficient and $\beta \in [0.7, 0.8]$ otherwise. In the next paragraphs we fix $b = 10\%$ and $\beta = 0.5$ for graphs with clustering coefficients below 0.2, and $\beta = 0.75$ otherwise. 

{\bf Node classification.} Using the provided node classes, we train a logistic regression model with the node embeddings as features. For the recommended values for $b$ and $\beta$, we observe modest yet statistically significant improvements for six out of the eight graphs. Additionally, these improvements remain consistent for larger training sizes.  In Appendix~\ref{sec:app_hyperparameters}, we discuss how smoothing enhances the quality of embeddings for low-degree nodes without compromising those of high-degree nodes. An example is shown in Figure~\ref{fig:node_clf_cora} where we see that even if smoothing yields only small overall accuracy gains for Cora when evaluated on all nodes, the representation of nodes of lower degree improves.  

\begin{figure}[h!]
\centering
\includegraphics[width=60mm]{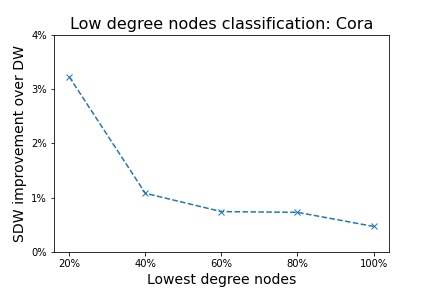}
  \caption{Macro-F1 scores for node classification on Cora for $\beta=0.75$ when evaluating embeddings for low-degree nodes.}
  \label{fig:node_clf_cora}
\end{figure}

{\bf SmoothNode2Vec (SN2V).}
%
%
Many different random walk strategies have been proposed and a natural question is whether some of these algorithms also achieve less skewed pair frequency distribution. We chose node2vec as a representative approach that allows biased random walk strategies with well-understood properties. Even if ``outward-biased'' walk strategies appear to have some smoothing effect, the pair frequency distribution still follows a power law, see appendix~\ref{sec:app_data} for details. 

We evaluate node2vec and its smooth version for combinations of the hyperparameters $p$ and $q$ from \{0.25, 0.5, 2, 4\}.  
For SmoothNode2Vec we again observe impressive improvements for link prediction for Pubmed, Git, Deezer and Flickr, independent of $p$ and $q$.  For node classification, smoothing considerably improves upon node2vec when  node2vec performs poorly otherwise the improvements are very small yet statistically significant.
All results can be found in Appendix~\ref{sec:app_n2v}.
%
%

{\bf Comparison with other embedding methods.}
We compare the performance of SmoothDeepWalk (SDW) and SmoothNode2Vec (SN2V) to other approaches using the recommended default values for $b=10\%$ and $\beta=0.5$ if the graph's clustering coefficient is below 0.2, otherwise we set $\beta=0.75$. For node2vec's hyperparameters we set $p=4, q=0.25$ because for these values node2vec has an overall strong performance. 

For our direct competitor, the smooth negative sampling algorithm by ~\cite{neg_smooth}, denoted as SNS, 
we set the negative smoothing parameter to 0.5 which corresponds to $\beta=0.5$ for SDW. The results are worse than ours and the SNS sampling approach is very slow,  the average time increase is more than 30 times compared to SDW, see details in Appendix~\ref{sec:app_competitors}. We also evaluate the adaptive negative sampling approach~\cite{adaptive_neg_sampling} which penalizes high degree nodes and is a heuristic that has a smoothing effect on the pair frequencies. This approach also has a very high time complexity. 

We consider several representative time efficient embedding methods such as LINE~\cite{line},  NetMF~\cite{matfac}, Walklets~\cite{walklets}, Verse~\cite{verse}, and the more recent residual2vec~\cite{res2vec}. We provide details about the methods and their specific hyperparameters in appendix~\ref{sec:app_competitors}. We keep all common hyperparameters identical (number of samples, embedding dimension, walk length, etc.) and use otherwise the default values for the hyperparameters of the corresponding methods. For each of the eight graphs we compute accuracy scores achieved by the nine methods. We provide all numeric results for the eight individual graphs in Appendix~\ref{sec:app_competitors}. 

In order to evaluate the overall performance of the different algorithms, we calculate the average and the lowest of the scores compared to the best score achieved by any of the embedding algorithms. The minimum score shows how stable is a given approach across different graphs.
In Table~\ref{tab:aggregate} we see that both SmoothDeepWalk and SmoothNode2Vec consistently improve upon DeepWalk and node2vec. For link prediction the two smoothing approaches are by far better than all competitors. While SmoothDeepWalk and SmoothNode2Vec rarely yield the absolute best scores for node classification, their performance remains consistently stable across all graphs. Nonetheless, it should be noted that with respect to micro-F1 scores the two smoothing methods are not the best ones which underlines that smoothing improves upon all node classes. 


However, we want to emphasize that pair frequency smoothing is a versatile technique that can be applied to many of the aforementioned algorithms, which should be viewed not as competitors, but as potential partners. For example, Walklets and Verse generate positive pairs from random walks and NetMF explicitly factorizes the  implicitly defined matrix from Theorem~\ref{thm:matfac}. 

\begin{table}\caption{Aggregated precision@100 scores for link prediction (left) and macro-F1 scores for node classification (right) for different algorithms.} \label{tab:aggregate}
\centering
\scriptsize
\begin{tabular}{ccc}
\multicolumn{3}{c}{Precision@100}\\
\toprule
       method & mean  &  min\\
       \toprule
	  {\bf SN2V} & 0.924 & 0.538 \\
	  \midrule
	 {\bf SDW} & 0.897 & 0.471 \\
	\midrule
	 N2V & 0.779 & 0.279 \\
	\midrule
	 DW & 0.75 & 0.204 \\
	 \midrule
	 Verse & 0.455 & 0.078 \\
	 \midrule
	 res2vec & 0.43 & 0.053 \\
	\midrule
	 NetMF & 0.376 & 0.002 \\
	 \midrule
	 Line & 0.365 & 0.073 \\
	\midrule
	 Walklets & 0.223 & 0.01 \\
		\bottomrule
\end{tabular}
\hspace*{3mm}
\begin{tabular}{ccc}
\multicolumn{3}{c}{Macro-F1}\\
\toprule
       method & mean  &  min\\
\toprule
 	{\bf SN2V} & 0.938 & 0.742 \\
	\midrule
	{\bf SDW} & 0.932 & 0.767 \\
	\midrule
	N2V & 0.926 & 0.734 \\
	\midrule
	DW & 0.916 & 0.742 \\
	\midrule
	NetMF & 0.912 & 0.66 \\
	\midrule
	Line & 0.902 & 0.708 \\
	\midrule
	Walklets & 0.901 & 0.599 \\
	\midrule
	Verse & 0.796 & 0.638 \\
	\midrule
	res2vec & 0.775 & 0.296 \\
		\bottomrule
\end{tabular}
\end{table}

\section{Conclusions and future work} \label{sec:concl}

We presented pair frequency smoothing, a novel regularization technique for (random) walk-based node embedding learning algorithms. Theoretical and experimental results highlight the potential of the proposed approach. 
A natural direction for future research is to apply the method to other node embedding algorithms like those discussed in the previous section. Additionally, various approaches to negative sampling have been developed; we refer to~\cite{adaptive_neg_sampling} for an overview of recent results. Future research could explore optimal ways to combine smooth sampling of positive pairs with these negative sampling approaches.

However, we have identified two main open questions:
\begin{itemize}[leftmargin=*]
\item[-] Develop a deeper understanding of how the graph structure affects the performance of smoothing. Our observations on the clustering coefficient are only a first step in this direction. For instance, we have yet to explain why smoothing enhances node classification on the Cora and Citeseer graphs but provides no advantage for link prediction on these graphs.
\item[-] Explore the use of smoothing in graph neural networks. In Section~\ref{sec:app_competitors} in the appendix we present results for the unsupervised version of GraphSage~\cite{graphsage} where the input are positive and negative node pairs generated in the same way as in DeepWalk. The gains are relatively modest. However, it should be noted that this setting does not allow GNNs to show their full potential and DeepWalk and node2vec yield somewhat better results. Nonetheless, we anticipate that smooth sampling can be used as a tool for sampling from the local neighborhood of each node. For example,  PinSage~\cite{pinsage} uses importance sampling from the local neighborhood of each node by employing short random walks. 
\end{itemize}


\clearpage
\bibliographystyle{plain}
\bibliography{smooth}

\newpage
\appendix

\section{Theoretical analysis} \label{sec:app_theory}
\subsection{Proofs} \label{sec:app_proofs}
\addtocounter{thmx}{-4}
\begin{thmx}
Let $G=(V, E)$ and $D$ be the corresponding random walk corpus. Let $\beta \in (0, 1]$.  The following hold for SmoothDeepWalk:
\begin{itemize}[leftmargin=*]
\item Let $M_\beta \ge c T_\beta(T_\beta + 1)$ for some $c\ge1$.  With probability at least $1-e^{-c}$, SmoothDeepWalk needs between $T_\beta -1$ and $T_\beta +1$ passes over the corpus $D$.
\item Let $S_{u,v}$ be the number of positive samples of pair $u, v$ returned by SmoothDeepWalk. It holds $\mathbb{E}(S_{u,v}) = T_\beta\#(u, v)^\beta$. If $\#(u,v) \ge 1/\varepsilon^2 \log 1/\delta$ for $\varepsilon, \delta \in (0, 1)$, then with probability $1-\delta$ it holds $S_{u,v} = T_\beta(\#(u,v)^{\beta} \pm \varepsilon \#(u, v))$.
\end{itemize}
\end{thmx} 
\begin{proof}
In a single pass over the pair corpus $D$ we expect to sample $M_\beta$ pairs. We thus expect $T_\beta = \lceil M/M_\beta \rceil$ passes over $D$ will be necessary in order to sample $M$ pairs. We first bound the probability to sample less than $M$ pairs in $T_\beta + 1$ passes over $D$.  Let $X_i$ be indicator random variables such that $X_i=1$ if the $i$-pair is sampled, for $1 \le i \le M(T_\beta +1)$ and let $L = \sum_{i=1}^{M(T_\beta +1)} X_i$. It holds $\mathbb{E}(L) = M + M_\beta$. The $X_i$ are independent but not identically distributed Bernoulli random variables. By Hoeffding's inequality we obtain
\[\Pr(L \le M) = \Pr(L \le \mathbb{E}(L) - M_\beta) \le \] \[ \exp\left(-\frac{2M_\beta^2}{ M(T_\beta +1)}\right) =\exp\left(-\frac{2M_\beta}{ T_\beta(T_\beta +1)}\right)  \le e^{-2c}\] where the last inequality follows from the assumption $M_\beta \ge c T_\beta(T_\beta + 1)$.

Similarly, we obtain that \[\Pr\left( \sum_{i=1}^{M(T_\beta -1)} X_i \ge M \right) \le e^{-2c}\] By the union bound we obtain that the number of passes over $D$ is between $T_\beta -1$ and $T_\beta +1$ with probability $1 - 2 e^{-2c} \ge 1 - e^{-c}$.

We next analyze the number of samples for individual pairs. In the following we assume that the number of passes is concentrated around $T_\beta$ with high probability.
For a pair $u, v$,  there are $T_\beta \#(u,v)$ sampling trials with probability $\#(u, v)^{\beta-1}$, therefore the expected value of $S_{u,v}$ follows from linearity of expectation. The random numbers $r \in \mathbb{U}[0, 1)$ are independent, and we have identically distributed variables. We can thus consider $\#(u, v)$ Bernoulli i.i.d. random variables $Y_i$ with $\mathbb{E}(Y_i) = \#(u,v)^{\beta-1}$. By Chernoff inequality for $\#(u,v) \ge 1/\varepsilon^2 \log 1/\delta$ we have 
\[\Pr \left(\sum_{i=1}^{\#(u, v)} Y_i \ge \#(u,v)^\beta + \varepsilon \#(u, v) \right) \le\] \[\exp \left(-2 \#(u, v) \varepsilon^2 \right) \le \delta/2\]
Similarly we have \[\Pr \left( \sum_{i=1}^{\#(u, v)} Y_i \le \#(u,v)^\beta - \varepsilon \#(u, v)\right) \le \delta/2\] and by the union bound the concentration bound follows. 
\end{proof}
\\\\\\
In the following proofs we use the following bounds on the Zipfian distribution over $P$ unique pairs.
For $z\neq 1$ it holds \[S_z = \sum_{i=k}^P \frac{1}{i^z} = O\left( \int_{x=k}^P \frac{1}{x^z} d x\right) = O\left(\frac{x^{1-z}}{1-z} \Big|_{x=k}^{x=P}\right)\]
for $1 \le k \le P$. In particular, for $k=1$ and $z>1$ we have $S_z = C_z$ for some constant $C_z>1$ that monotonically decreases with increasing $z$. For $k>1$ and $z>1$ it holds $\sum_{i=k}^P = O(k^{1-z})$. For $z<1$ it holds $S_z = O(\frac{1}{1-z}P^{1-z})$. 

\begin{thmx}
Let $D$ be a corpus of cardinality $M$ of node pairs. Let the frequencies in $D$ follow a Zipfian distribution with parameter $z \ge 0$, and $D_\beta$  be the $\beta$-smoothed corpus for $\beta \in (0, 1]$. Let $j$ be the minimum pair rank such that $f({D_\beta})_j > f({D})_j$. 
\begin{itemize}
\item If $z>1$ and $\beta z > 1$, then $j = c(z, \beta)$ for some constant $c$. 
\item If $z>1$ and $\beta z < 1$, then $j = P^{\frac{1-\beta z}{z -\beta z}} = o(P)$ where $P$ is the number of unique pairs in $D$.
\item If $z<1$, then $j = O(P)$. 
\end{itemize}
\end{thmx}
\begin{proof}\let\qed\relax
\hspace{1mm}We analyze for which node pairs the probability to be sampled increases after applying smoothing when the pair frequencies follow a Zipfian distribution. We need $T_\beta = \frac{\sum_{i=1}^P f_i}{\sum_{i=1}^P f_i^\beta} = M/M_\beta$ passes over the random walk corpus $D$. The $j$-th pair that is originally sampled $f_j$ times, after smoothing will be sampled in expectation $T_\beta f_j^\beta$ times, and as we have shown in Theorem~\ref{thm:sampling} the number of samples is concentrated around the expected value. Thus, the $j$-th pair will be sampled more times when
\begin{equation}
f_j^{1-\beta} < \frac{\sum_{i=1}^P f_i}{\sum_{i=1}^P f_i^\beta}
\end{equation}

Independent of the smoothing exponent, we sample exactly $M$ positive pairs in total. Under the power law assumption we have $f_j = \frac{X_z}{j^z}$ for some $z\ge 0$ for the frequency of the $j$-th pair, and it holds $\sum_{j=1}^P \frac{X_z}{j^z} = M$. By replacing the power law frequencies $f_j$ by $X_z/j^z = \frac{M}{S_z j^z}$ in the above we obtain 
\[\sum_{i=1}^P f_i^\beta =\sum_{i=1}^P \frac{X_z^\beta}{i^{\beta z}} = \sum_{i=1}^P \frac{M^\beta}{S_z^\beta i^{\beta z}} = M^\beta S_{\beta z} / S_z^\beta \]
Plugging in the above into (1) we obtain
\[f_j^{1-\beta} < \frac{\sum_{i=1}^P f_i}{\sum_{i=1}^P f_i^\beta} \Leftrightarrow \left(\frac{M}{S_z j^z}\right)^{1-\beta} < \frac{M S_z^\beta}{M^\beta S_{\beta z}}\]

\[\Leftrightarrow \frac{M^{1-\beta}}{S_z^{1-\beta} j^{z(1-\beta)}} < \frac{M^{1-\beta} S_z^\beta}{S_{\beta z}} \Leftrightarrow j^{z(1-\beta)} > S_{\beta z}/S_z\]
 We can now observe how smoothing affects the number of generated positive samples for less frequent pairs using the bounds for $S_z$. We distinguish following cases:
\begin{itemize}[leftmargin=*]
\item $z>1$, $\beta z > 1$. For all pairs of rank at least $j$ where $j^{z(1-\beta)} > C_{\beta z}/C_z$, we will sample more positive pairs when smoothing with parameter $\beta$. It holds $C_{\beta z} > C_z$. However, only for a constant fraction of the most frequent pairs we will generate less positive pairs, as even for $\beta z$ close to 1 the term $C_{\beta z}$ is a small constant.
\item $z>1, \beta z < 1$. The minimum necessary rank $j$ now depends on the total number of pairs $P$:  $j^{z(1-\beta)} > \frac{P^{1-\beta z}}{(1-\beta z) C_z}$. It holds $\beta \in [0, 1/z)$.  Assuming $1-\beta z$ is constant,  we need $j > P^{\frac{1-\beta z}{z -\beta z}}$. Thus, for larger $z$ or for a small enough $\beta$ only $o(P)$ pairs will have a larger number of positive samples before smoothing. 
\item $z<1$. Here we obtain that the minimum pair rank $j$ needs to satisfy $j^{z(1-\beta)} >  \frac{1-z}{1-\beta z}P^{z(1 - \beta)}$. This is only possible for a constant fraction of the positive pairs. This confirms the intuition that smoothing would significantly affect only embedding learning for highly skewed distributions.  (In the extreme case  when $z \rightarrow 0$ and the frequencies are almost uniformly distributed, smoothing will not change anything.) \hspace*{58mm} $\square$
\end{itemize} 
\end{proof}

\begin{thmx}
Let $D_\beta$ be $\beta$-smoothed corpus of cardinality $M_\beta$ for $\beta \in (0, 1]$. Let $\mu: V \rightarrow [0, 1]$ be the negative sampling probability distribution.  For nodes $u, v \in V$, SmoothDeepWalk optimizes the objective $\vec{u}^T \vec{v} = \log \frac{\#(u, v)^\beta}{\#u^{(\beta)} \mu(v)} - \log k$.
\end{thmx}
\begin{proof}
For the training objective it holds
\[J = \sum_{u \in V}\sum_{v \in V} T_\beta \#(u, v)^\beta \log \sigma(\vec{u}^T \vec{v}) +\] 
\[\sum_{u \in V}\sum_{v \in V} T_\beta\#(u, v)^\beta k \mathbb{E}_{x\sim \mu}[\log \sigma(-\vec{u}^T \vec{x})] =\]
\[T_\beta\sum_{u \in V}\sum_{v \in V} \#(u, v)^\beta \log \sigma(\vec{u}^T \vec{v})  +\] 
\[T_\beta \sum_{u \in V} \#(u)^{(\beta)}k \mathbb{E}_{x\sim \mu}[\log \sigma(\vec{u}^T \vec{x})]\]
$T_\beta$ is a constant independent of $u$ and $v$, thus for a fixed node pair $u,v$ we obtain 
\begin{multline}
J(u, v) = T_\beta  \#(u, v)^\beta \log \sigma(\vec{u}^T \vec{v}) + \\T_\beta k \#u^{(\beta)}\mu(v) \log \sigma(-\vec{u}^T\vec{v})
\end{multline}
By setting $y=\vec{u}^T\vec{v}$ and computing the derivative with respect to $y$, we obtain 
\[J(u, v) = T_\beta\#(u,v) \sigma(-y) - k T_\beta\#u^{(\beta)} \mu(v)\sigma(y)\] The expression has its maximum for 
\[y= \log \left(\frac{T_\beta\#(u,v)^\beta}{kT_\beta \#u^{(\beta)} \mu(v)}\right) = \log \left(\frac{\#(u,v)^\beta}{\#u^{(\beta)} \mu(v)}\right) - \log k\] 
\end{proof}

\begin{thmx}
Let the pair frequencies follow Zipfian distribution with parameter $z$. For $z>1$, {\sc Frequent} returns the exact frequency of the $k$ most frequent pairs using a sketch of size $b=O(k)$. For $z<1$, the required sketch size is $b=O(k^zP^{1-z})$ where $P$ is the number of unique pairs. 
\end{thmx}
\begin{proof}
Pairs that are not recorded in the sketch have frequency at most $M/b$, any other pair would have positive frequency. 
For $z<1$, we need a sketch size $b$ such that $\frac{M}{S_z k^z} > \frac{M}{b}$, thus we need  $b > k^z S_z$. By plugging the bounds on the Zipfian distribution defined above, we obtain $b=O(k^z)$ for $z>1$, and $b = O(k^z P^{1-z})$ for $z<1$ in order to compute the exact frequency of the most frequent $k$ pairs. 

For the case $z>1$ one can improve the space complexity to $O(k)$. Let $\mathbf{f}$ denote the frequency vector of the pairs in the corpus $D$. Assume the values in $\mathbf{f}$ are sorted in descending order and denote by $\mathbf{f}_{res(k)}$ the frequency vector after setting the first $k$ values in $\mathbf{f}$ to 0. Let $\|\mathbf{f}\|_1 = \sum_{i=1}^P f_i$. The key observation is that the approximation error can be expressed in terms of the residual norm $\|\mathbf{f}_{res(k)}\|_1$. 
The sum of counters after processing the sketch can be expressed as $\|\mathbf{f}\|_1 - \ell(b+1)$ where $\ell$ is the number of cases where a new pair caused the decrement of $b$ counters. Thus, the final value for the $i$th pair is at least $f_i - \ell$. Consider only the first $k$ values in $\mathbf{f}$, i.e., the cardinality of the $k$ most frequent pairs. It holds \[\|\mathbf{f}\|_1 - \ell(b+1) \ge \sum_{i=1}^k (f_i - \ell)\]  \[\|\mathbf{f}_{res(k)}\| \ge \ell (b-k+1)\] \[\ell \le \frac{\|\mathbf{f}_{res(k)}\|}{b-k+1}\]

For $z>1$ we can bound $\mathbf{f}_{res(k)} = \sum_{i=k}^P X_z i^{-z} = c\cdot X_z k^{1-z}$ for some constant $c>1$. Therefore for the frequency of the top $k$ pairs after processing the sketch it holds \[f_i - \ell \ge f_i - \frac{\|\mathbf{f}_{res(k)}\|}{b-k+1} =  f_i - \frac{c\cdot X_z k^{1-z}}{b-k+1}\] For a sketch of size $b = (2 c +1) k$ we thus obtain that the frequency of the $i$-th most frequent pair is \[f_i - \frac{X_z k^{-z}}{2} \ge X_zk^{-z} - \frac{X_zk^{-z}}{2} > 0\] for $i \le k$.
This implies that the $k$ most frequent pairs will have positive counters after processing the corpus and will be provably recorded in the sketch after the first pass.
\end{proof}

\subsection{Analysis for softmax optimization of the objective function}
 
The discussion so far has focused on using negative sampling. Negative sampling is an efficient approximation of the softmax objective and opens the door to a rigorous theoretical analysis~\cite{matfac,word2vec_expl}, hence it has become the standard technique when training node embeddings.  However, hierarchical softmax can be used instead of negative sampling~\cite{word2vec} and a natural question to ask is if it would suffer from the same drawbacks. The cross-entropy objective for a node pair $u, v$ can be written as 
\[J_\theta(u, v) = -\vec{u}^T\vec{v} + \log \sum_{x \in V} \exp(\vec{u}^T\vec{x})\] where $\theta$ are all learnable node embeddings.

As shown in~\cite{bengio_senecal}, the gradient of the above can be simplified to

\[\nabla J_\theta(u, v) = -\nabla_\theta \vec{u}^T\vec{v} + \sum_{x \in V} \mu(x|u) \nabla_\theta \vec{u}^T\vec{x}\]
where $\mu(x|u) = \frac{\exp{(\vec{u}^T\vec{x}})}{\sum_{u \in V}\exp{(\vec{u}^T\vec{x}})}$ is the softmax probability of $x$ given $u$. Thus, by computing the derivative with respect to a given $u$ we will also update $u$'s embedding in the direction of $x$ by $\delta \mu(x|u)x$ where $\delta$ is the learning rate. The difference to negative sampling is that here we weight the corresponding gradient updates according to the softmax probability given by the current embeddings instead of generating negative examples. Thus, the problem is present also here: a heavy pair $u, v$ of rank $i$ would imply that we force the embeddings for a pair $u, x$ of rank $j>i$ to become less similar. 

\section{Code} \label{sec:app_code}

In Algorithm~\ref{alg:sketch} we show pseudocode for the {\sc Frequent} algorithm applied to a corpus of node pairs generated in a streaming fashion. We remind again that the sketch $S$ is implemented as a hash table and the amortized processing time per pair is constant. Algorithm~\ref{alg:est} shows how pair frequencies are estimated by {\sc Frequent}. For pairs recorded in the sketch we return the exact frequency, otherwise we return an overestimation of the frequency. Observe that the overestimation results in decreasing the sampling probability for infrequent pairs in Algorithm~\ref{alg:main}.

\begin{algorithm2e}[t!]
\small
\caption{Pair frequency sketching}\label{alg:sketch}
\KwInput{Graph $G$ over $n$ nodes, random walk length $\ell$, window size~$t$, sketch budget $b$}
\KwOutput{A data structure $S$ such that $S[(u,v)] \approx \#(u, v)$ for all $(u, v) \in \mathcal{C}_G$ {\bf define} }
sketch $S=\emptyset$ \hspace*{2mm}\Comment{\scriptsize the sketch is denoted by $S$}
\For{$u \in G$}{
 { Generate random walk $T$ of length $\ell$ starting at $u$}\\
 \For{$i = 1 \text{ to } \ell$ \hspace*{1mm}\Comment{\scriptsize skip-gram pair generation}}{
 	 $u = T[i]$\\
	 $start = \max(0, i-t)$\\
	 $end = \min(\ell, i+t)$\\
	 \For{$j \in [start:i, i+1:end]$}{
	 	$v = T[j]$\\
	 	\If{$(u, v) \in S$ \hspace*{1mm}\Comment{\scriptsize pair already in sketch, update pair counter}}{
				$S[(u, v)] \pluseq$ 1
				}
		\Else{$S[(u, v)]=1$ \hspace*{1mm}\Comment{\scriptsize add pair to sketch}}
		\If{$|S|> b$ \hspace*{4mm}\Comment{\scriptsize sketch already full}}{
				\For{$(u, v) \in S$ \hspace*{1mm}\Comment{\scriptsize decrease counter for pairs in sketch}}{
					$S[(u, v)]\minuseq$ 1 \\
					\If{$S[(u, v)]==0$ \hspace*{1mm}\Comment{\scriptsize remove pairs with counter 0}}{
							delete $(u, v)$ from $S$
							}
				} 
		}
	 	}
 	}
}
 \Return $S$
\end{algorithm2e}

\begin{algorithm2e}[t!]
\small
\caption{Frequency estimation}\label{alg:est}
\KwInput{Sketch $S$, pair $(u, v)$, default value $w$}
\If{$(u, v) \in S$}{
\Return $S[(u, v)]$
}
\Else{
\Return $w$ 
}
\end{algorithm2e}

\section{Extensions of the approach} \label{sec:app_extensions}

\paragraph{``Unsmoothing'' with $\beta > 1$.} 
The original motivation behind the paper is to extend the set of node pairs whose similarity should be preserved. However, in certain applications it might be desirable to focus on the most frequent pairs. The algorithm can be adjusted to handle such cases by setting $\beta > 1$. Since $\#(u,v)^{\beta-1} \ge 1$, we need an upper bound $U$ on the pair frequencies and then we can sample with probability $\frac{\#(u, v)^{\beta-1}}{U^{\beta-1}} \le 1$. The expected number of sampled pairs per one pass over the corpus $D$ is $M_\beta/U^{\beta-1}$. Note that $M_\beta > M$ for $\beta > 1$ and unsmoothing is useful in cases when the data is not very skewed and $U \ll M$. Thus, the number of necessary passes over $D$ is small.    

\paragraph{Using other sketching approaches.} 
Observe that by overestimating the frequency of infrequent pairs we decrease the sampling probability for such pairs. Alternatively, we can use  Count-Sketch~\cite{count_sketch} that yields an unbiased estimate of the frequency of all pairs. However, Count-Sketch is a randomized approximation algorithm that introduces an (unbiased) error for all pairs. It requires the computation of the median of $p$ independent estimates for a robust estimation. Even if $p$ is usually a small constant, our prototype implementation shows that this considerably increases the overall running time.   

\section{Experimental evaluation} \label{sec:app_experiments}

\begin{table*}
\scriptsize
\centering
\caption{Fraction of the of total number pairs in corpus generated by the $k\%$ most frequent pairs for $1 \le k \le 10$.}
\label{tab:app_skew}
\begin{tabular}{l  cccccccccc}
\toprule 
 & 1\% & 2\% & 3\% & 4\% & 5\% & 6\% & 7\% & 8\% & 9\% & 10\% \\
\toprule
Cora & 0.487 & 0.609 & 0.671 & 0.712 & 0.742 & 0.766 & 0.785 & 0.801 & 0.814 & 0.825\\
\midrule
Citeseer & 0.489 & 0.64 & 0.723 & 0.773 & 0.806 & 0.83 & 0.849 & 0.864 & 0.877 & 0.888\\
\midrule
Pubmed & 0.459 & 0.546 & 0.601 & 0.642 & 0.673 & 0.698 & 0.719 & 0.736 & 0.75 & 0.761 \\
\midrule
Git & 0.35 & 0.418 & 0.46 & 0.491 & 0.515 & 0.532 & 0.544 & 0.553 & 0.559 & 0.562\\
\midrule
Flickr & 0.325 & 0.369 & 0.388 & 0.398 & 0.404 & 0.408 & 0.409 & 0.409 & 0.409 & 0.409\\
\midrule
Deezer & 0.414 & 0.489 & 0.535 & 0.568 & 0.594 & 0.615 & 0.632 & 0.646 & 0.658 & 0.667\\
\midrule 
Wiki & 0.361 & 0.459 & 0.52 & 0.559 & 0.587 & 0.61 & 0.627 & 0.641 & 0.651 & 0.66\\
\midrule
Blog & 0.27 & 0.364 & 0.43 & 0.482 & 0.522 & 0.553 & 0.578 & 0.595 & 0.607 & 0.614\\
\bottomrule
\end{tabular}
\end{table*}

In this section we provide all details for the observations outlined Section~\ref{sec:experiments}. The section is organized as follows. In Section~\ref{sec:app_data} we discuss in more details the datasets and the experimental setting. In particular, we show that the assumptions about data skew are justified. In Section~\ref{sec:app_hyperparameters} we analyze the performance of SmoothDeepWalk with respect to computational complexity and the utility of the trained embeddings for node classification and link prediction. In particular, we study how the performance is affected by the newly introduced hyperparameters: the sketch budget $b$ and the smoothing exponent $\beta$ and provide a deeper discussion about the  recommended default values for $b$ and $\beta$. In Section~\ref{sec:app_competitors} we present results for  applying smooth pair sampling to other random walk based embedding algorithms and compare the performance of SmoothDeepWalk to other embedding methods.
\subsection{Datasets and overall setting}~\label{sec:app_data} 

\paragraph{Graph datasets.} We selected real-life graphs from different domains with different properties. The three citation networks have been widely used as benchmarks for machine learning on graphs, Git, Flickr and Deezer are large sparse graphs, and Wikipedia and BlogCatalog are denser networks with a large clustering coefficient. Also, the number of node classes is different among the graphs and for Wikipedia and BlogCatalog we have a multilabel setting. 

\paragraph{Computational setting and hyperparameters.} 
We implemented pair smoothing in Python and performed experiments on a commodity laptop with a Ryzen 7 4.7 GHz CPU with 8 physical cores and 32 GB of memory. For training embeddings, we use a TensorFlow generator for smooth pair sampling, and a shallow neural network model with an embedding layer and a sigmoid output layer for classification of positive and negative pairs.  We generate batches of 1,000 positive pairs for the smaller graphs, and 10,000 for the larger ones. We use Adam as optimizer with an initial learning rate of 0.01. Otherwise, we use the hyperparameters from~\cite{node2vec}: embedding dimension $d=128$, random walks of length 80, 10 walks per node, 5 negative samples per positive pair, and negative pair sampling using the smoothed node degree distribution with $\alpha=0.75$ discussed in Section~\ref{sec:prel}. 

\paragraph{Reproducibility and fair evaluation}
We set the random seed for \texttt{numpy, random} and \texttt{TensorFlow}. Also, we use the same random walk corpus for DeepWalk and SmoothDeepWalk, and for node2vec and SmoothNode2Vec. We apply the same train/test splitting for the evaluation of all approaches. When comparing with other approaches we set the hyperparameters to be identical to the extent possible. 
We provide a prototype Python implementation together with detailed instructions on how to reproduce the results in a virtual environment \footnote{\url{https://github.com/konstantinkutzkov/smooth_pair_sampling}}. 


\begin{figure*}[t]
\centering
\includegraphics[width=60mm]{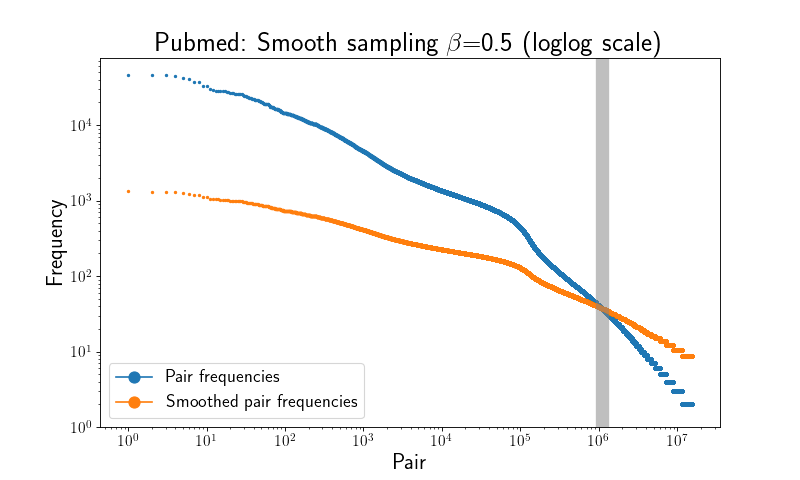}
\includegraphics[width=60mm]{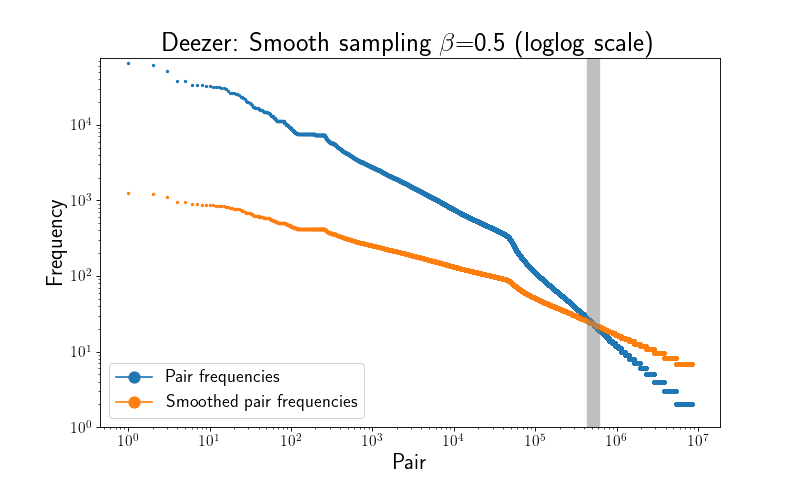}
\\
\includegraphics[width=60mm]{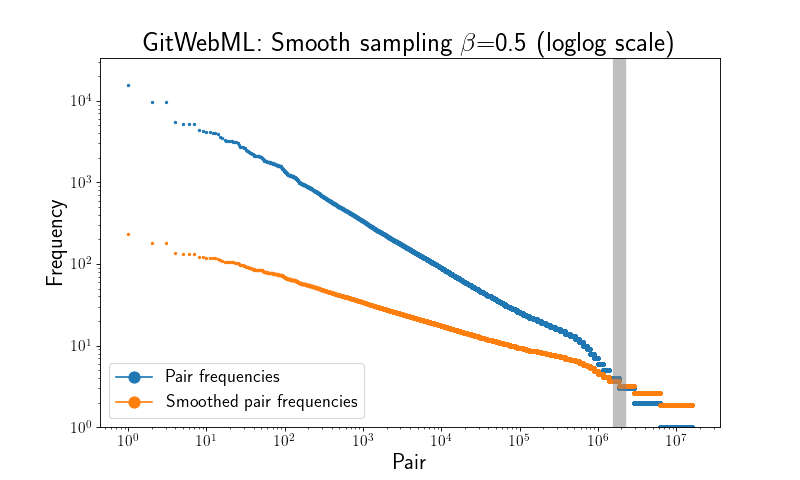}
\includegraphics[width=60mm]{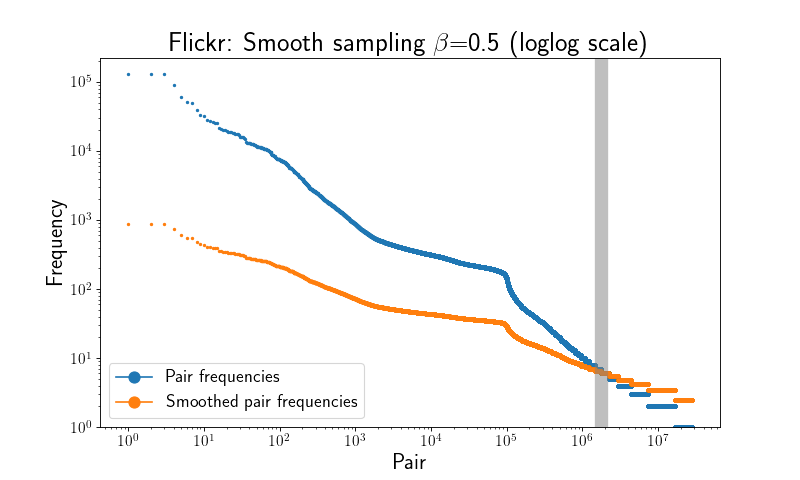}
\\
\includegraphics[width=60mm]{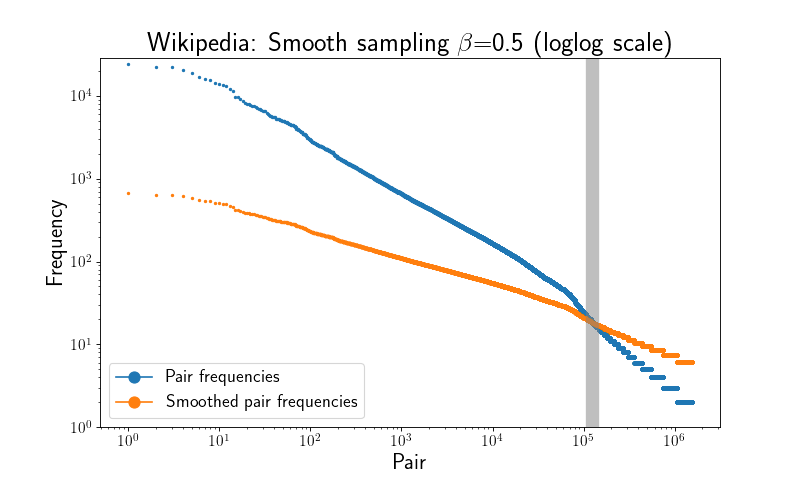}
\includegraphics[width=60mm]{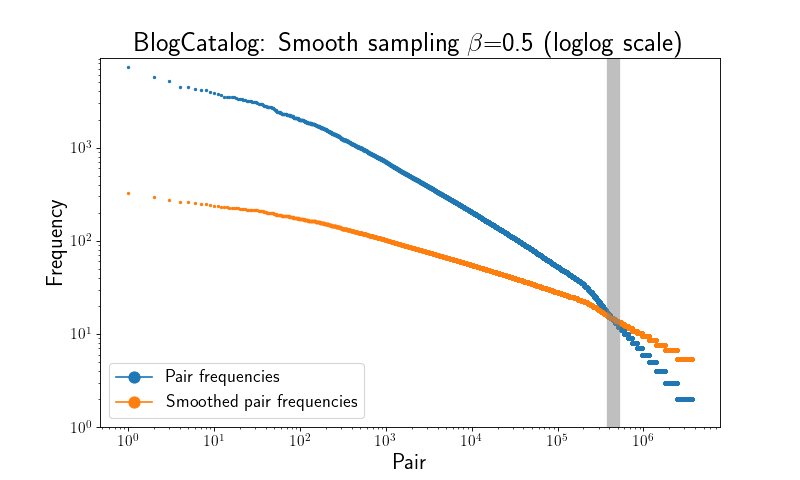}
\caption{The effect of smoothing the pair frequencies for the other six graphs (in addition to Figure~\ref{fig:smoothing_effect} in the main paper.). The grey vertical line shows the transition point after which smoothing leads to more positive samples for the corresponding pairs.}
\label{fig:app_smoothing_effect}
\end{figure*}

\paragraph{Skewed pair frequency distribution and the effect of smoothing.}
The original motivation of the paper is that the frequency distribution of positive pairs generated from random walk sequences is highly skewed and this can negatively affect the quality of the learned embeddings. 
Figure~\ref{fig:app_smoothing_effect} shows that the frequency distribution of positive pairs generated by DeepWalk is indeed highly skewed for all graphs (in addition to Figure~\ref{fig:smoothing_effect} in the main paper). In Table~\ref{tab:app_skew} we show that a small fraction of all positive pairs dominate the overall frequency distribution. Except for the very sparse Flickr, we see that for the other seven graphs 5\% of the unique pairs generate more than 50\% of the total number of pairs.

\paragraph{Pair frequency distribution for node2vec random walk corpora}
In Figure~\ref{fig:app_smoothing_effect_n2v} we show that node2vec also results in a highly skewed pair frequency distribution. Note that we use the values $p=4, q=0.25$ which results in the least skewed frequency distribution, yet it still follows a power law.

\begin{figure}[t!]
\centering
\includegraphics[width=60mm]{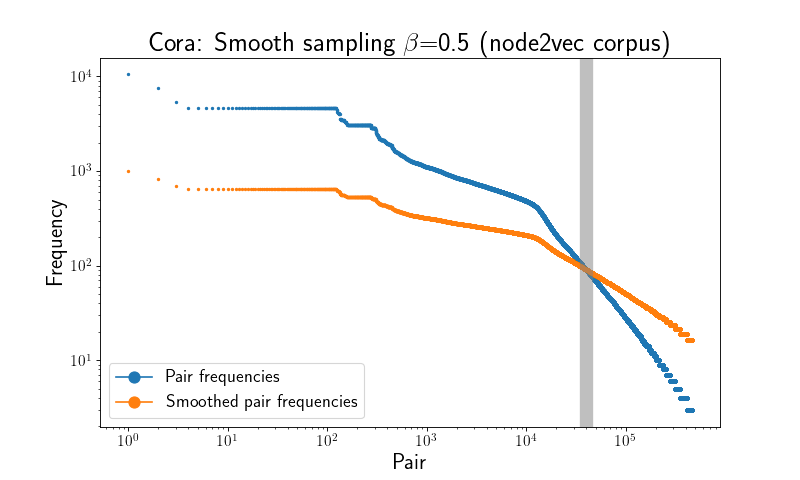}
\includegraphics[width=60mm]{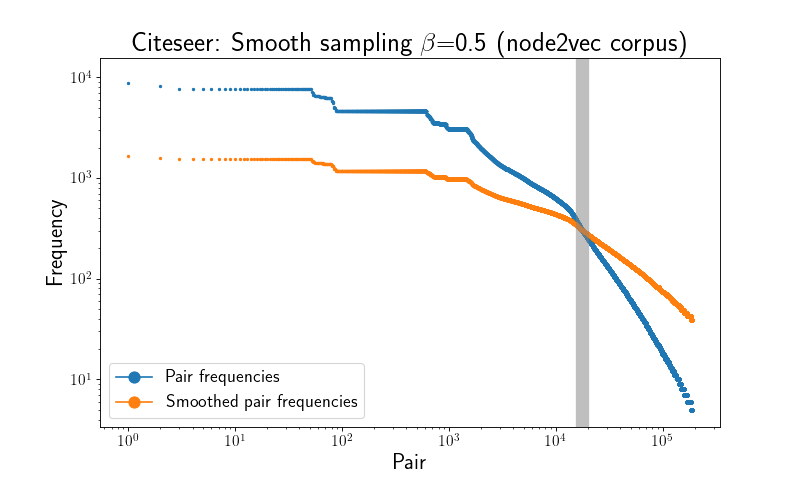}
\includegraphics[width=60mm]{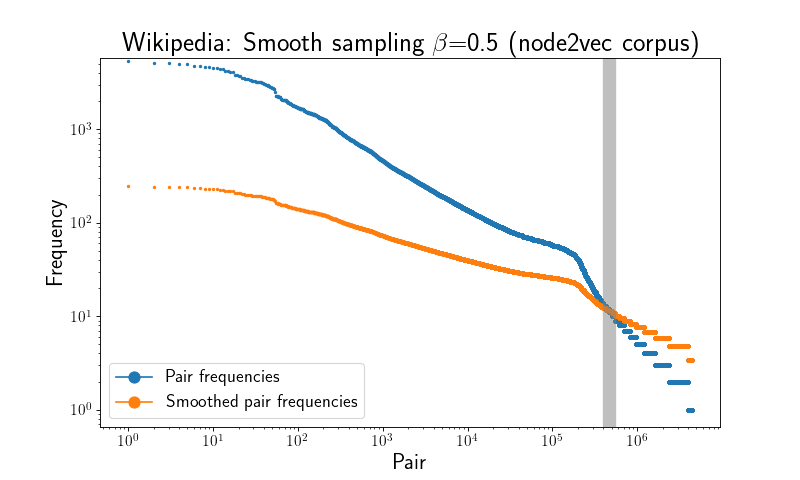}
\includegraphics[width=60mm]{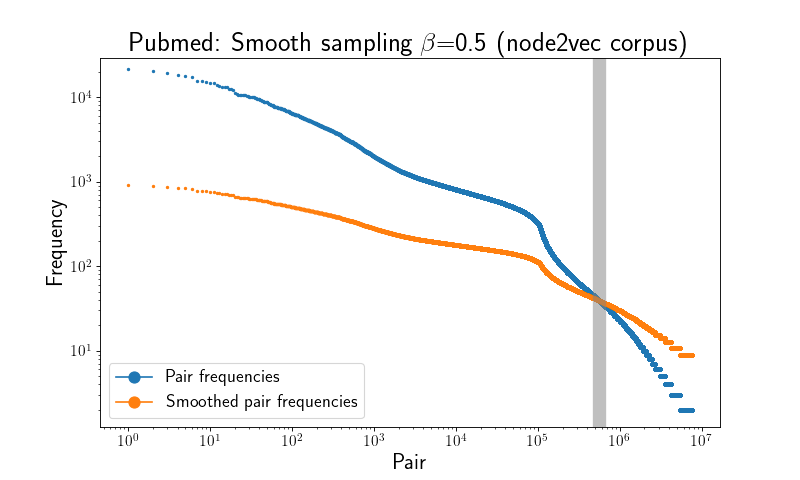}
\caption{The effect of smoothing the pair frequencies for node2vec($p=4, q=0.25$) corpora.}
\label{fig:app_smoothing_effect_n2v}
\end{figure}

\paragraph{Statistically significant positive pairs}
Following the discussion in Section~\ref{sec:why}, we next show the effect frequency smoothing has on the embedding learning process.
For the node pairs recorded in the sketch, using a budget of 10\% of all unique pairs $P$, we computed the fraction of pairs for which the difference between positive and negative samples is statistically significant at level $0.05$ according to a binomial test at probability 2/3, i.e., the number of positive occurrences of a pair is at least twice the number of negative occurrences. We argue that for such pairs the embedding algorithm learns to preserve the similarity between the corresponding nodes. As we see in Figure~\ref{fig:app_stat_sig} for Cora, Citeseer and Pubmed, by decreasing $\beta$ the fraction of significant pairs converges to 100\% which confirms the discussion in Section~\ref{sec:why} that smoothing serves as a frequency regularizer. 
\begin{figure}[h!]
\centering
\includegraphics[width=60mm]{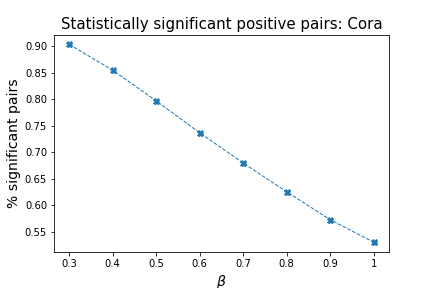}
\includegraphics[width=60mm]{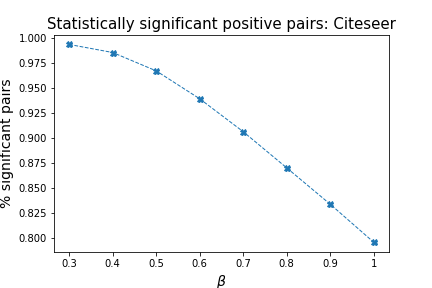}
\includegraphics[width=60mm]{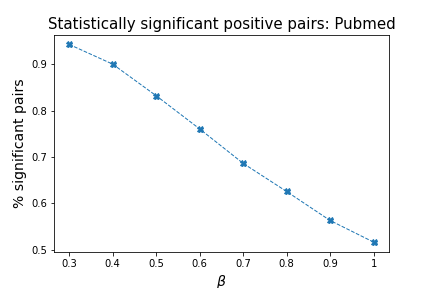}
\caption{Percentage of statistically significant positive pairs among pairs recorded in the sketch for varying~$\beta$.}
\label{fig:app_stat_sig}
\end{figure}

\subsection{Performance of SmoothDeepWalk} \label{sec:app_hyperparameters}
{\bf Running time.} In Figure~\ref{fig:app_rtimes} we show the effect smoothing has on the training time. We analyze the increase in running time depending on the smoothing level $\beta$ and the sketch budget $b$. We remind that a smaller sketch budget leads to a higher overestimation of pair frequencies not recorded in the sketch, and thus lower $\beta$ and~$b$ decrease the sampling probability and increase the total running time. We refer to Table~\ref{tab:app_neg_smooth} for the exact values for all graphs.   For the smaller graphs the increase is more significant. We conjecture the reason is the amount of fast CPU cache such that for larger graphs we observe more cache misses and the computational bottleneck is the embedding training by the binary classification model. Such considerations however are beyond the scope of the paper. 
\begin{figure}[t]
\centering
\includegraphics[width=60mm]{rtime_varying_beta.png}
\includegraphics[width=60mm]{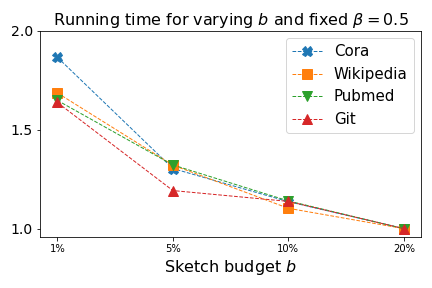}
\caption{Running time increases for various smoothing levels $\beta$ (left) and sketch budgets $b$ (right).}
\label{fig:app_rtimes}
\end{figure}

{\bf Link prediction.}
We use the experimental setting proposed in~\cite{hope} for link prediction. First, we delete 20\% of the graph edges at random such that the graph remains connected. Then we train node embeddings on the reduced graph. We sample 0.1\% of all $n \choose 2$ node pairs and 0.1\% of the removed edges. For the $k$ pairs with  the largest inner product among the sampled pairs we compute precision@$k$ and recall@$k$. We compute the mean of 100 independent trials, and we set a unique random seed for each of the 100 trials. 

The sparsity and low clusterability of Flickr makes link prediction much more challenging. However, independent of the clusterability structure and density, smoothing yields substantial gains for all six larger graphs apart from the two smaller graphs Cora and Citeseer.  But it should be mentioned that for dense graphs the gains are smaller. 
 
 \begin{figure*}[t]
 \centering
\includegraphics[width=60mm]{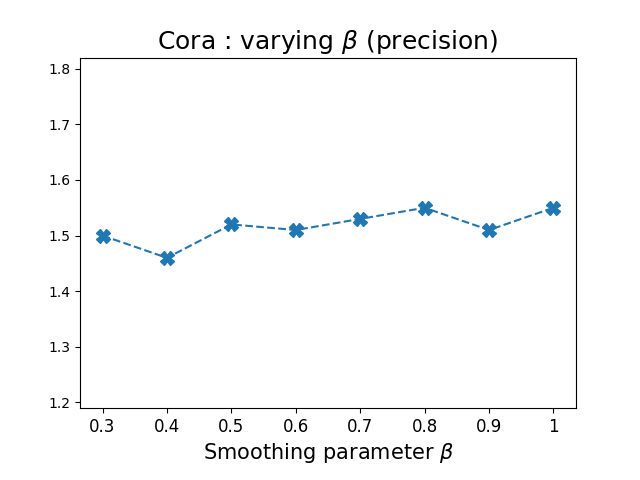}
\includegraphics[width=60mm]{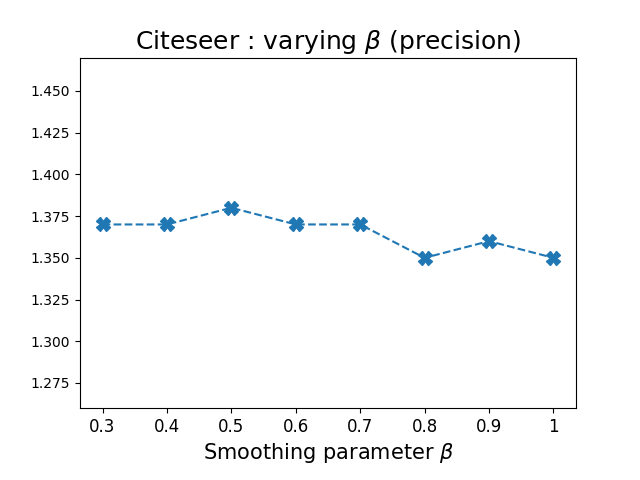} 
\\
\includegraphics[width=60mm]{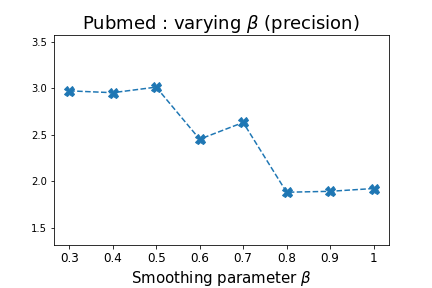}
\includegraphics[width=60mm]{Git_prec_linkpred_varying_beta.png} 
\\
\includegraphics[width=60mm]{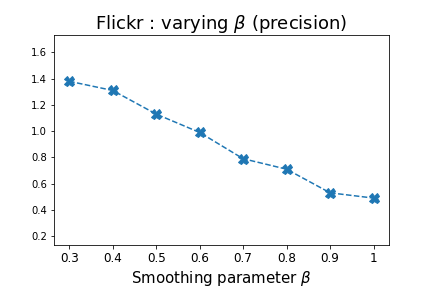}
\includegraphics[width=60mm]{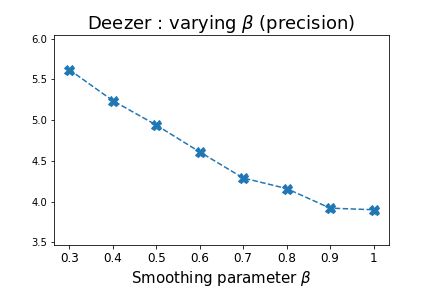}
\\
\includegraphics[width=60mm]{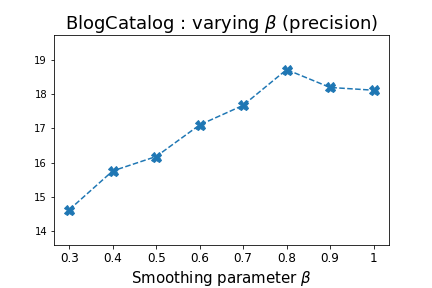}
\includegraphics[width=60mm]{Wikipedia_prec_linkpred_varying_beta.png}
\caption{The effect of varying $\beta$ for link prediction.}
\label{fig:app_beta_linkpred}
\end{figure*}

\begin{figure*}[t]
\centering
\includegraphics[width=48mm]{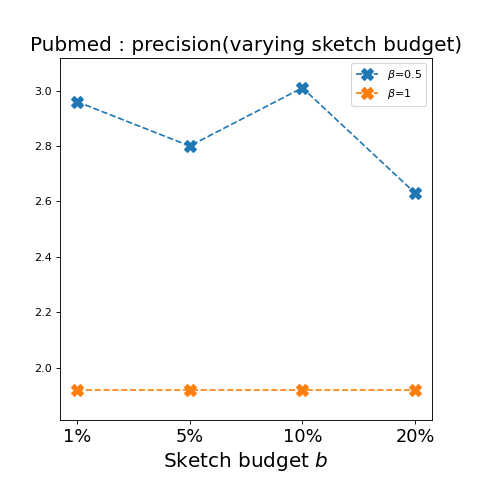} 
\includegraphics[width=48mm]{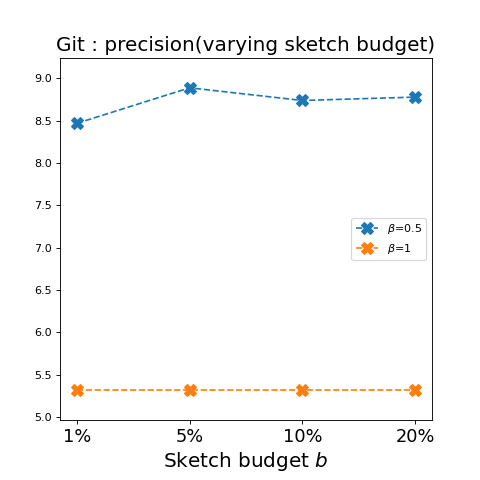} 
\includegraphics[width=48mm]{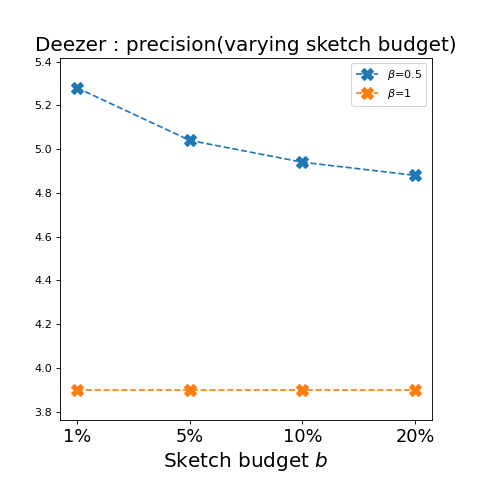} 
\\
\includegraphics[width=48mm]{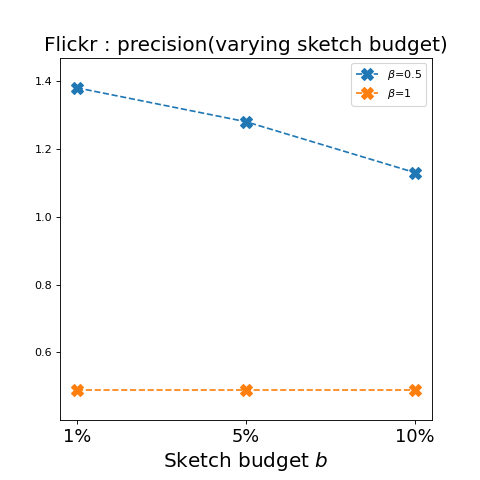} 
\includegraphics[width=48mm]{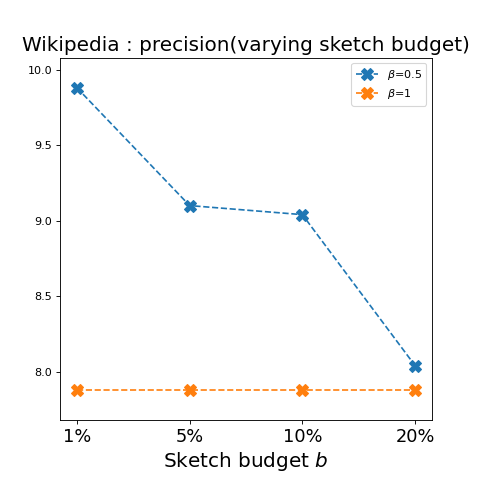} 
\includegraphics[width=48mm]{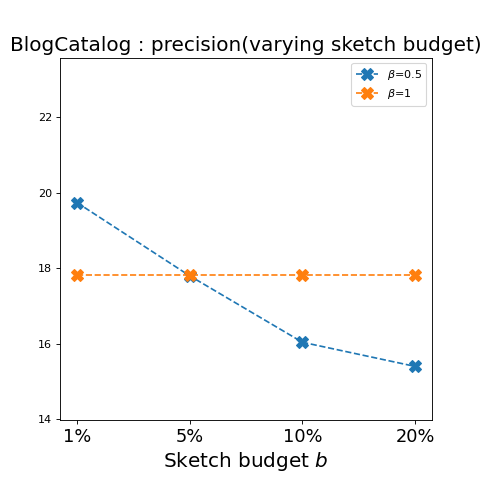} 
\caption{The effect of varying the sketch  budget $b$ for link prediction.}
\label{fig:app_sketchsize_linkpred_app} 
\end{figure*}

\paragraph{Node classification.} 
We train a logistic regression model with L2-regularization using the scikit-learn implementation. For the two multilabel classification problems, on Wikipedia and BlogCatalog, we train one-vs-rest classifiers using logistic regression as the base estimator\footnote{\scriptsize \url{https://scikit-learn.org/stable/modules/generated/sklearn.multiclass.OneVsRestClassifier.html}}.  
For each graph we run 100 independent runs by randomly sampling 10\% of the nodes for training and using the the rest 90\% for testing using scikit-learn's \texttt{train\_test\_split} function. We set a random seed for each of the 100 runs such that different algorithms work with the same train and test datasets by providing a random seed.  We report the mean score of the result. When comparing SmoothDeepWalk and DeepWalk we report results statistically significant according to a $t$-test at significance level 0.01. Results where the difference between the two methods is not statistically significant are shown in gray.

The improvements are consistent also for larger train sizes, see Figures~\ref{fig:app_clf_var_train_size_macro} and~\ref{fig:app_clf_var_train_size_micro} for macro-F1 and micro-F1 scores, respectively.

%

\paragraph{Node classification for low and high degree nodes}
 
It is reasonable to assume that low-degree nodes will be involved in less random walks than high-degree nodes. A natural question to ask is if smoothing improves the representation of such underrepresented nodes. In Figure~\ref{fig:app_low_high} we show the improvement in percentage of macro-F1 scores of SmoothDeepWalk over DeepWalk when training a model by considering only the lowest/highest degree nodes. After sorting the nodes according to their degree, we select the $t\%$ nodes of highest/lowest degree. Note that for 100\% we consider all nodes. We observe that we indeed improve the representation for low degree nodes, and never get worse results when considering only high degree nodes. This confirms the objective of smoothing, the emergence of ``middle class'' pairs by introducing a progressive tax-like regularization.  

{\bf The effect of $\beta$ and $b$.}  We evaluate SmoothDeepWalk for $\beta \in [0.3, \ldots, 1]$, $b=10\%$. 
In Figures~\ref{fig:app_beta_linkpred} and \ref{fig:app_beta}  we show the effect the smoothing parameter $\beta$ has on link prediction and node classification. The plots show the two main observations discussed in the main paper: 
\begin{enumerate}
\item The accuracy values have a functional dependency on $\beta$, with some small fluctuations. This opens the door to efficient approaches to hyperparameter optimization. 
\item For graphs with a low clustering coefficient, lower values of $\beta$ are beneficial, while for graphs with a high clustering too aggressive smoothing can be harmful. 
\end{enumerate}

%
%

Similarly, we  fix $\beta=0.5$ and evaluate SmoothDeepWalk for sketch budgets $b \in [1\%, 5\%, 10\%, 20\%]$ of the (estimated) number of unique pairs $P$. We remind that a smaller budget leads to a larger overestimation of the frequency of infrequent pairs which in turn decreases the sampling probability and corresponds to more aggressive filtering of infrequent pairs. Thus, $b$ and $\beta$ both influence which pairs are considered important.
In Figures~\ref{fig:app_sketchsize_linkpred_app} and~\ref{fig:app_clf_var_budget} we show the effect of varying the sketch budget $b$ for link prediction and on node classification. It appears the budget has little influence on node classification but smaller budgets, and thus more aggressive filtering of the most infrequent pairs, can have a positive effect for the denser graphs. (Note that we don't present results for Flickr with budget=20\% because we run out of memory.)

\begin{figure}
\centering
\includegraphics[width=57mm]{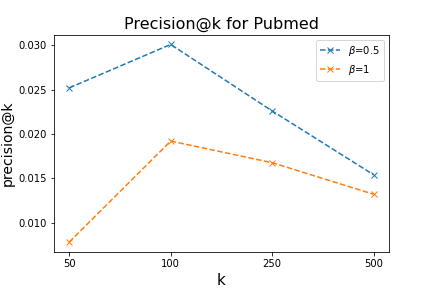} 
\includegraphics[width=57mm]{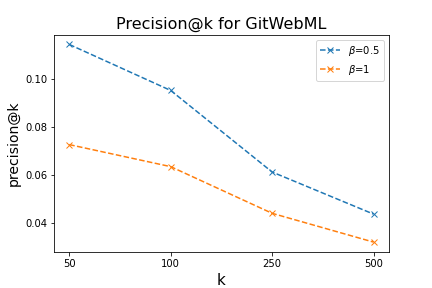} 
\includegraphics[width=57mm]{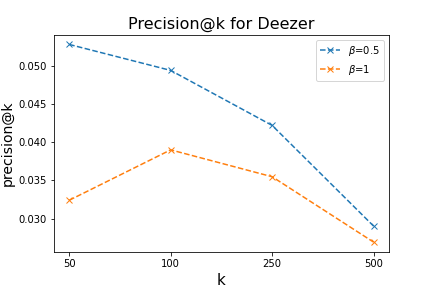} 
\includegraphics[width=57mm]{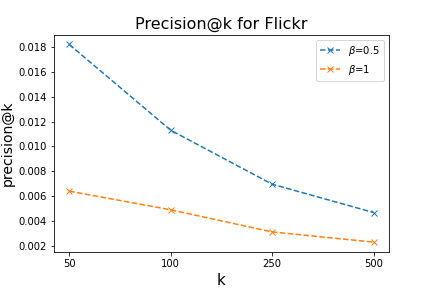} 
\includegraphics[width=57mm]{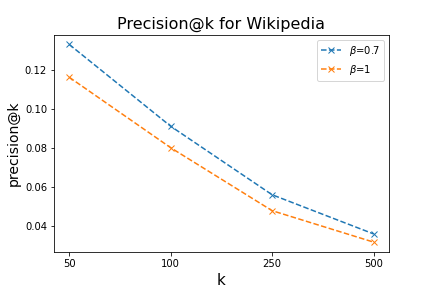} 
\includegraphics[width=57mm]{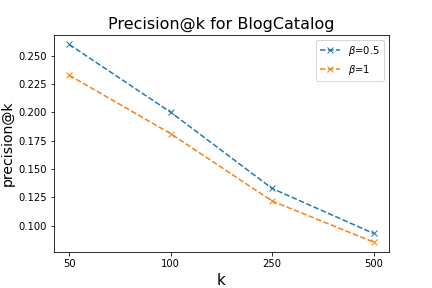} 
\caption{Precision@$k$ scores for varying $k$ for link prediction.}
\label{fig:app_prec_at_k}
\end{figure} 

\begin{figure}
\includegraphics[width=57mm]{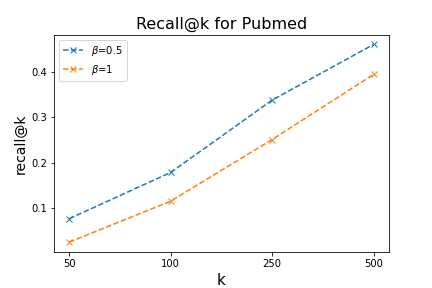} 
\includegraphics[width=57mm]{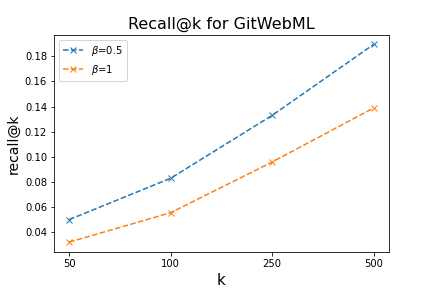} 
\includegraphics[width=57mm]{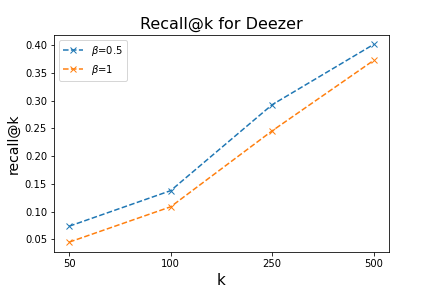} 
\includegraphics[width=57mm]{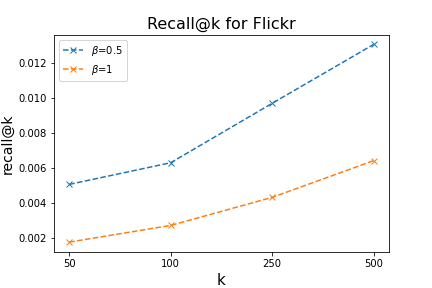} 
\includegraphics[width=57mm]{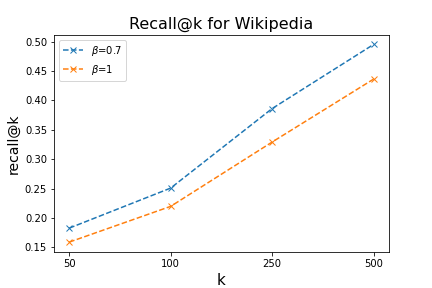} 
\includegraphics[width=57mm]{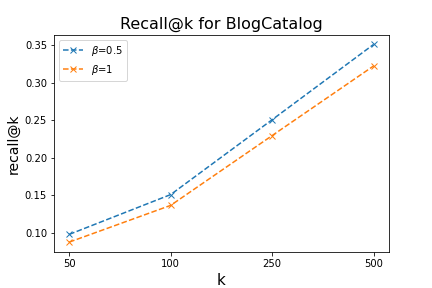} 
\caption{Recall@$k$ scores for varying $k$ for link prediction.}
\label{fig:app_recall_at_k}
\end{figure}

\begin{figure*}[h!]
\centering
\includegraphics[width=60mm]{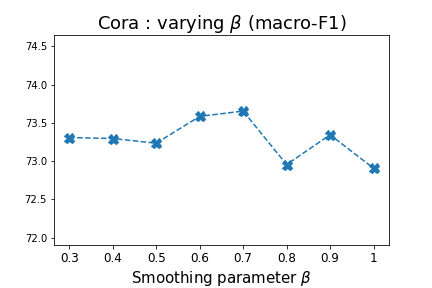} 
\includegraphics[width=60mm]{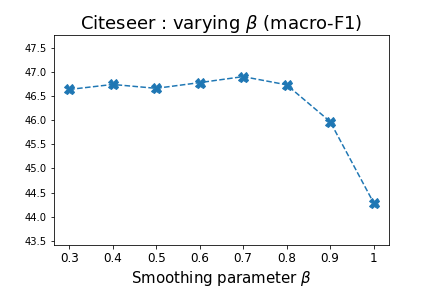}
\\
\includegraphics[width=60mm]{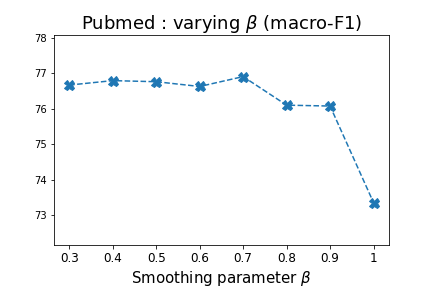}
\includegraphics[width=60mm]{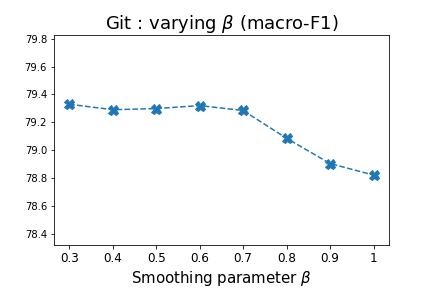}
\\
\includegraphics[width=60mm]{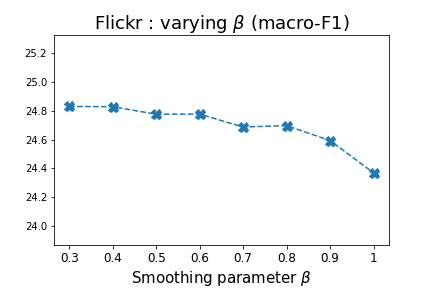} 
\includegraphics[width=60mm]{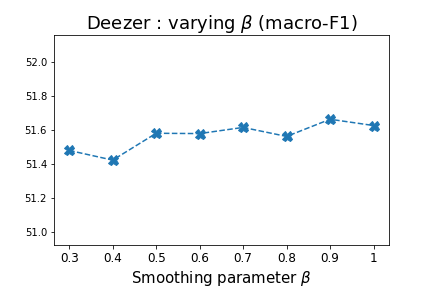} 
\\
\includegraphics[width=60mm]{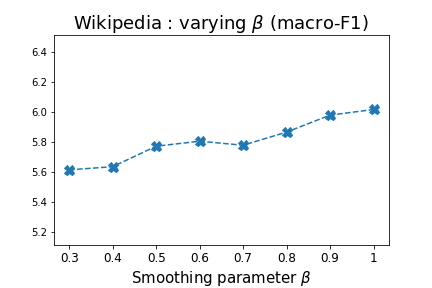}
\includegraphics[width=60mm]{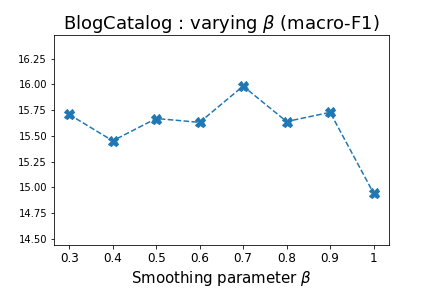}
\caption{The effect of varying $\beta$ for node classification.}
\label{fig:app_beta}
\end{figure*}

\begin{figure*}[t!]
\centering
\includegraphics[width=57mm]{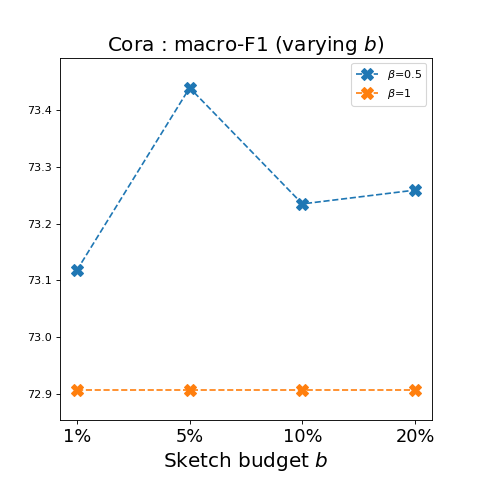}
\includegraphics[width=57mm]{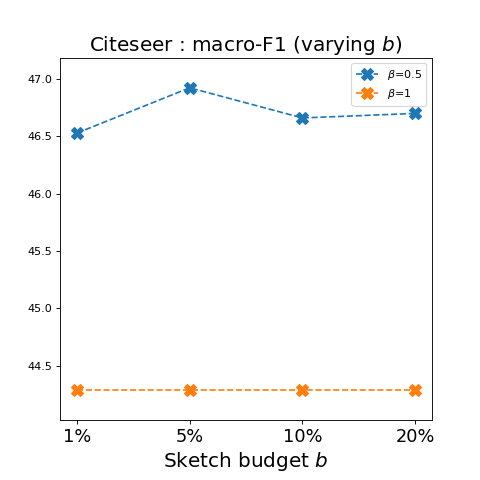}
\\
\includegraphics[width=57mm]{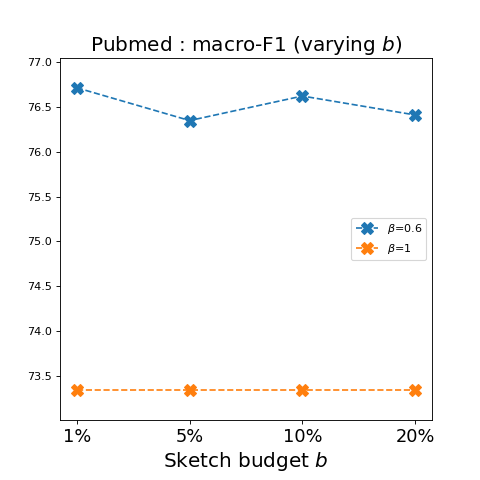}
\includegraphics[width=57mm]{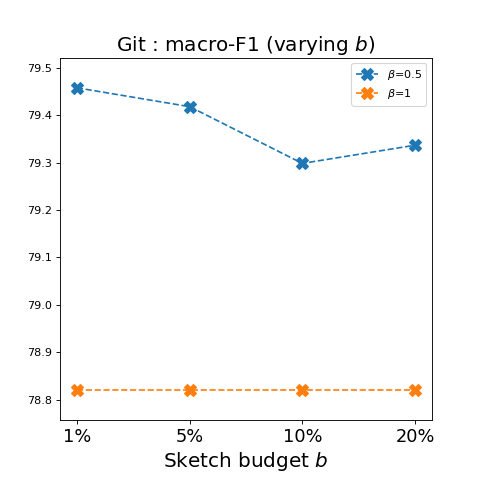}
\\
\includegraphics[width=57mm]{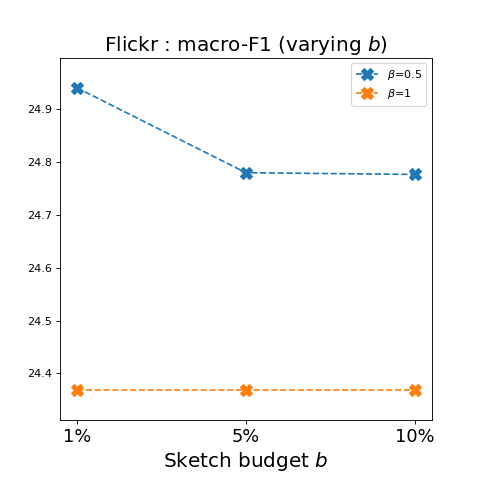}
\includegraphics[width=57mm]{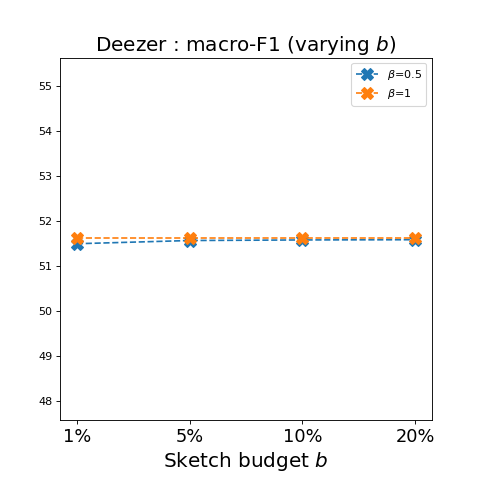}
\\
\includegraphics[width=57mm]{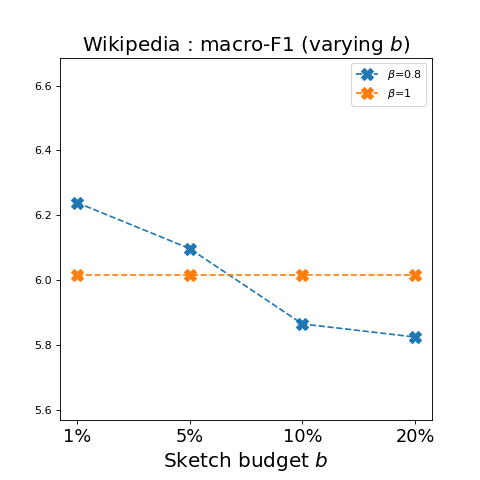}
\includegraphics[width=57mm]{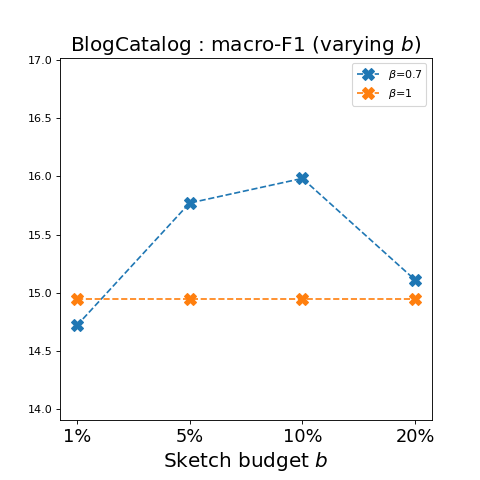}
\caption{The effect of varying the sketch  budget $b$ for node classification.}
\label{fig:app_clf_var_budget}
\end{figure*}

\begin{figure*}[t!]
\centering
\includegraphics[width=57mm]{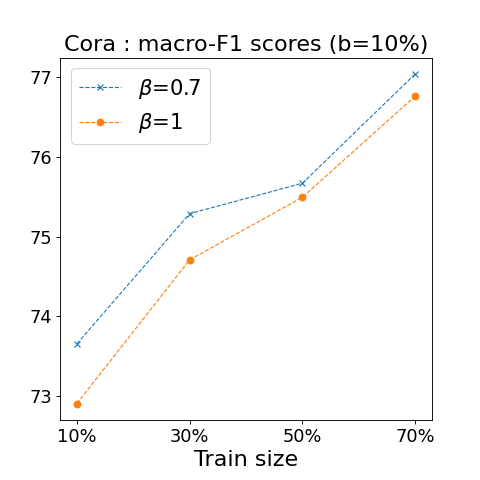}
\includegraphics[width=57mm]{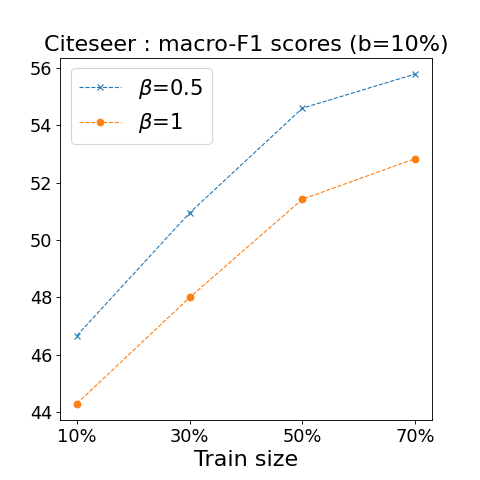}
\\
\includegraphics[width=57mm]{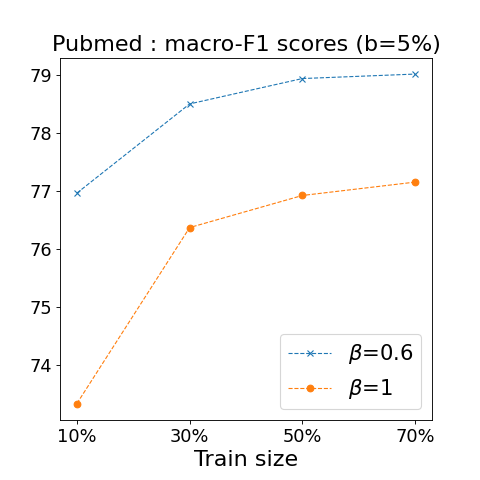} 
\includegraphics[width=57mm]{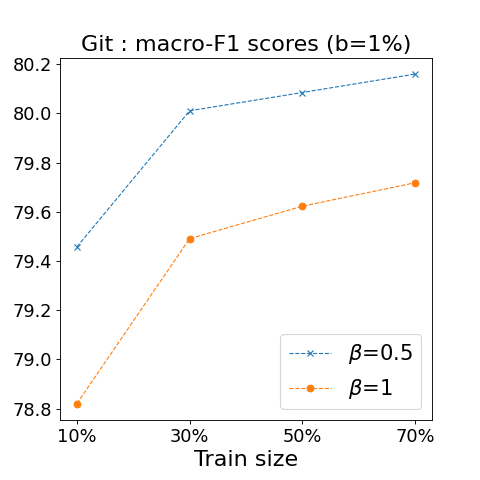} 
\\
\includegraphics[width=57mm]{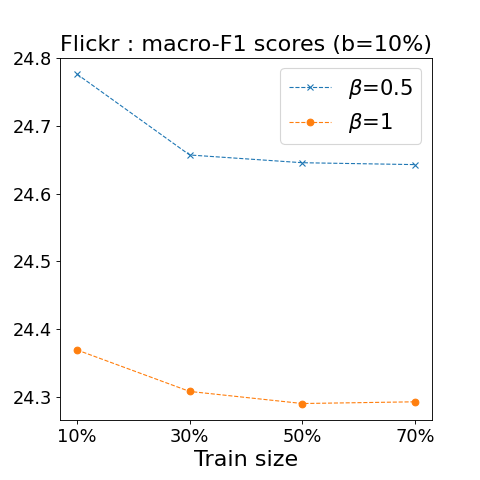}
\includegraphics[width=57mm]{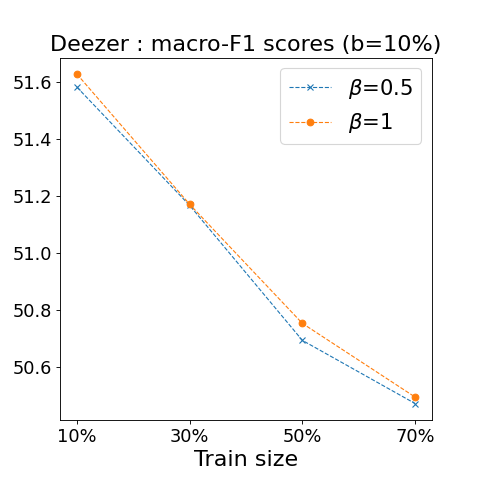}
\\
\includegraphics[width=57mm]{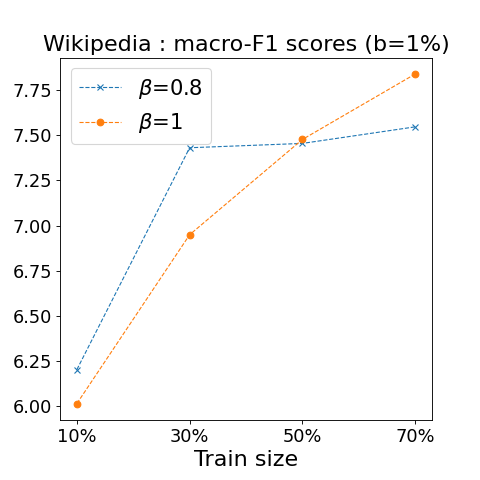} 
\includegraphics[width=57mm]{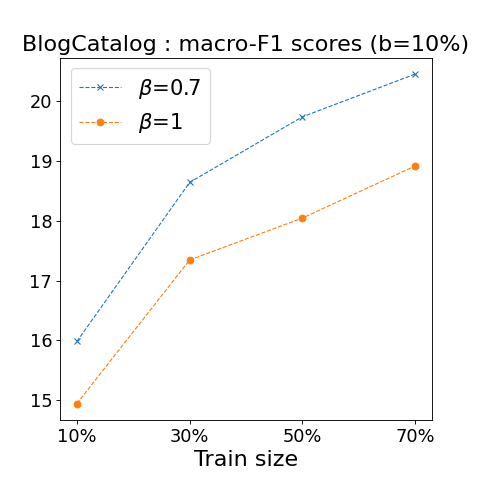} 
\caption{Macro-F1 scores for node classification for varying train sizes.}
\label{fig:app_clf_var_train_size_macro}
\end{figure*}

\begin{figure*}[t!]
\centering
\includegraphics[width=57mm]{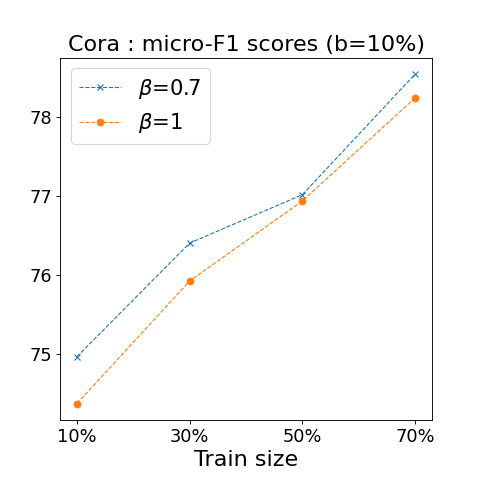}
\includegraphics[width=57mm]{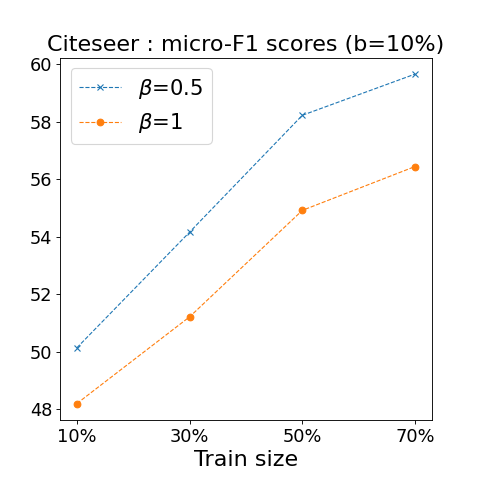}
\\
\includegraphics[width=57mm]{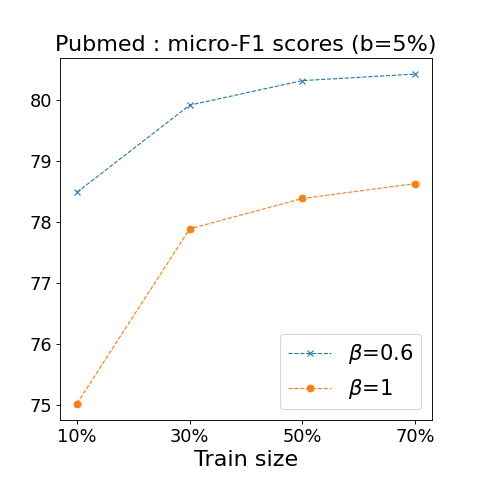} 
\includegraphics[width=57mm]{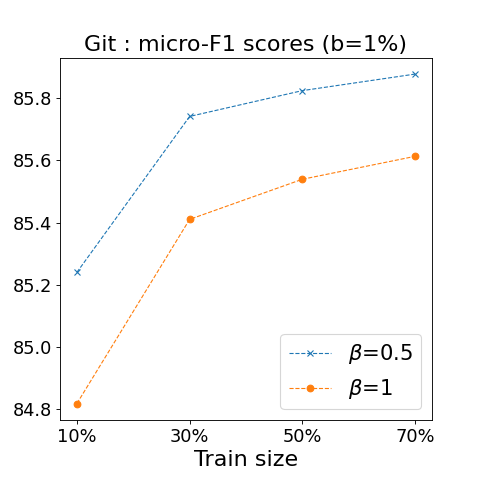} 
\\
\includegraphics[width=57mm]{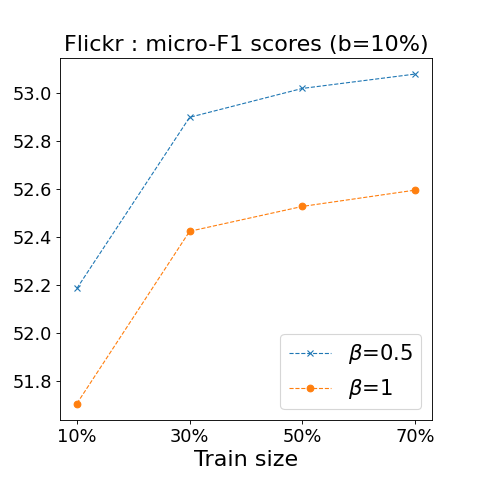}
\includegraphics[width=57mm]{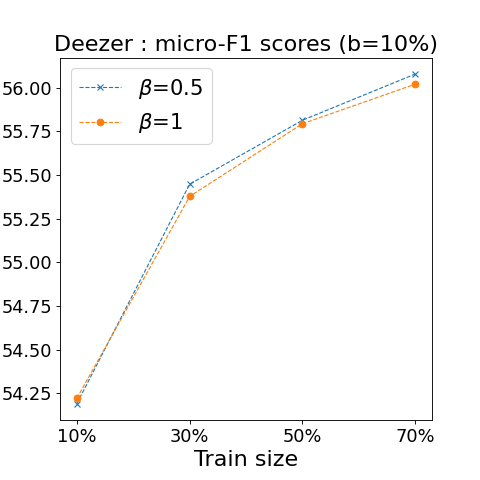}
\\
\includegraphics[width=57mm]{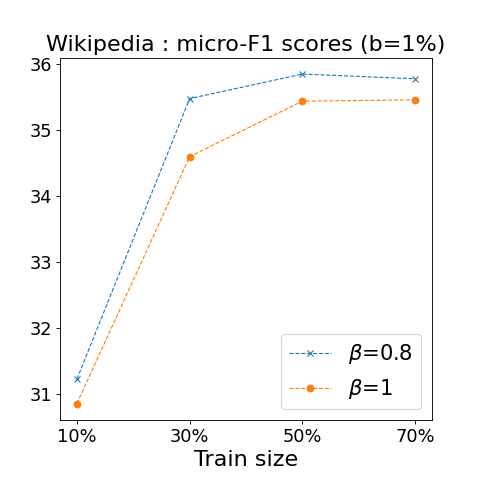} 
\includegraphics[width=57mm]{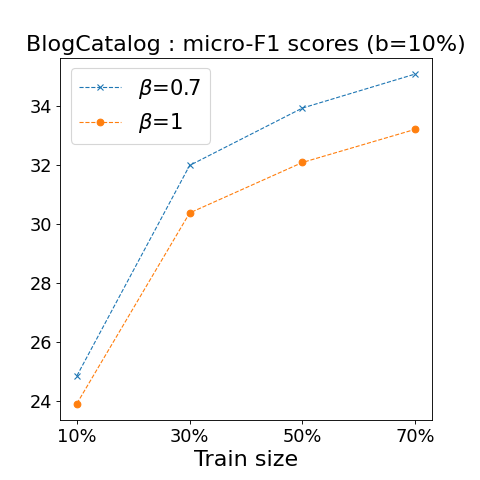} 
\caption{Micro-F1 scores for node classification for varying train sizes.}
\label{fig:app_clf_var_train_size_micro}
\end{figure*}

\begin{figure*}[h!]
\centering
\includegraphics[width=60mm]{Cora_low_degree_clf.png} 
\includegraphics[width=60mm]{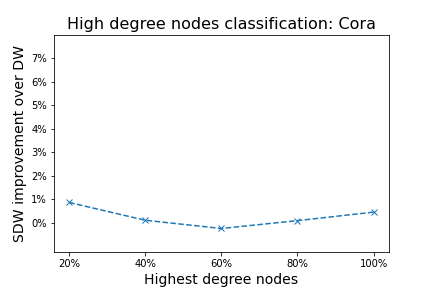} 
\\
\includegraphics[width=60mm]{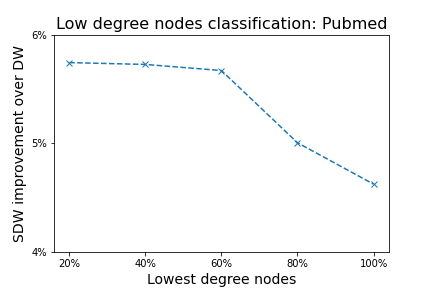} 
\includegraphics[width=60mm]{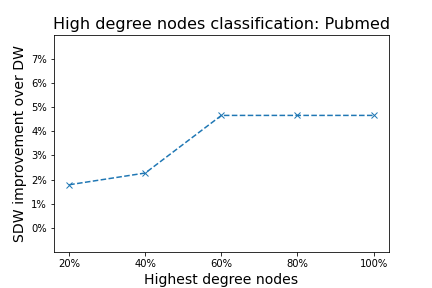}
\\ 
\includegraphics[width=60mm]{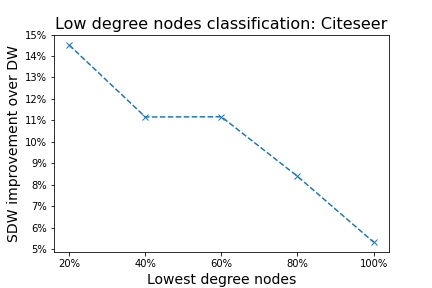} 
\includegraphics[width=60mm]{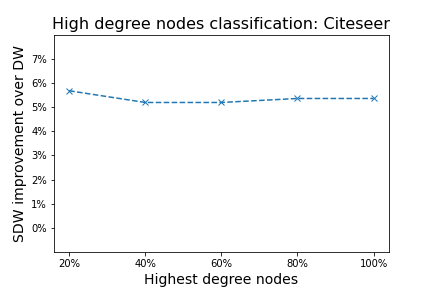}
\\ 
\includegraphics[width=60mm]{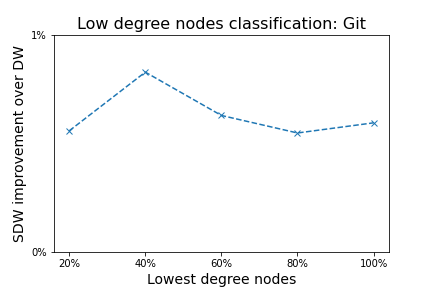} 
\includegraphics[width=60mm]{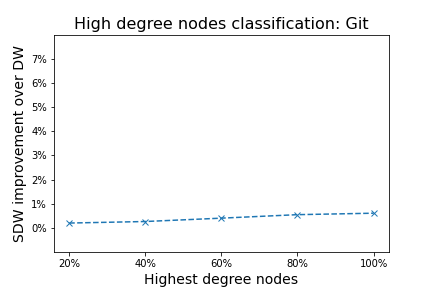} 
\caption{Macro-F1 scores for node classification for low and high degree nodes.}
\label{fig:app_low_high}
\end{figure*}


\subsection{SmoothNode2Vec} \label{sec:app_n2v}
Figure~\ref{fig:app_n2v_linkpred} shows the performance of SmoothNode2Vec for link prediction with the recommended hyperparameters for $\beta$ and $b$ for different values of the hyperparameters $p$ and $q$. We again observe significant improvements for the large sparse graphs, and smoothing again fails to yield any noticeable improvement for Cora and Citeseer. For node classification, as shown in Figure~\ref{fig:app_n2v_nodeclf}, smoothing considerably improves the accuracy in cases when node2vec performs poorly, otherwise the gains are incremental yet statistically significant. We observe that outward-biased random walks, i.e., large $p$ and small $q$, yield stronger performance for most graphs and we thus use $p=4, q=0.25$ as default values. 

\begin{figure*}[h!]
\centering
\includegraphics[width=60mm]{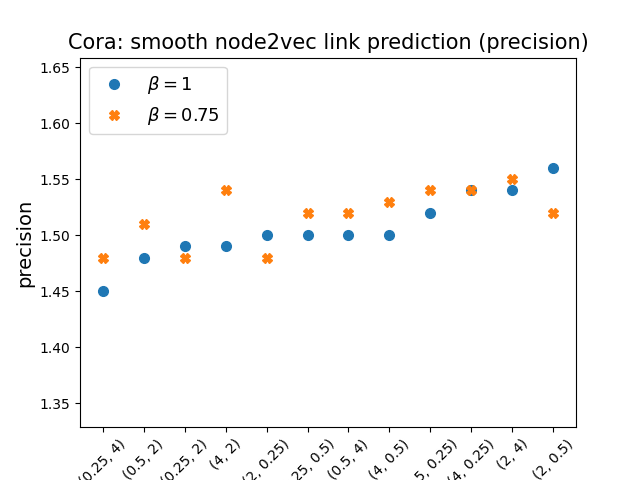} 
\includegraphics[width=60mm]{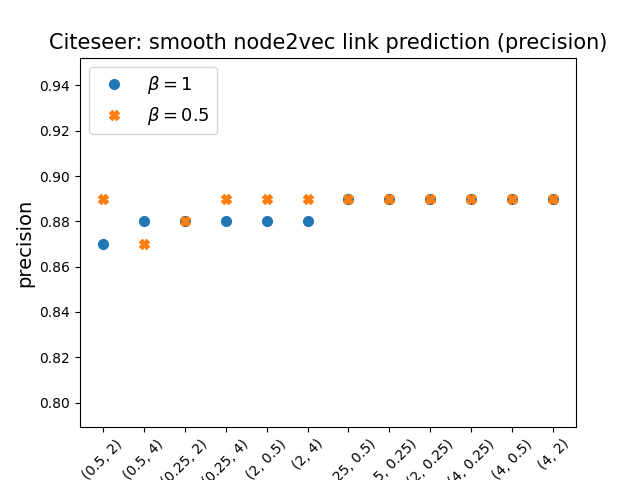} 
\\
\includegraphics[width=60mm]{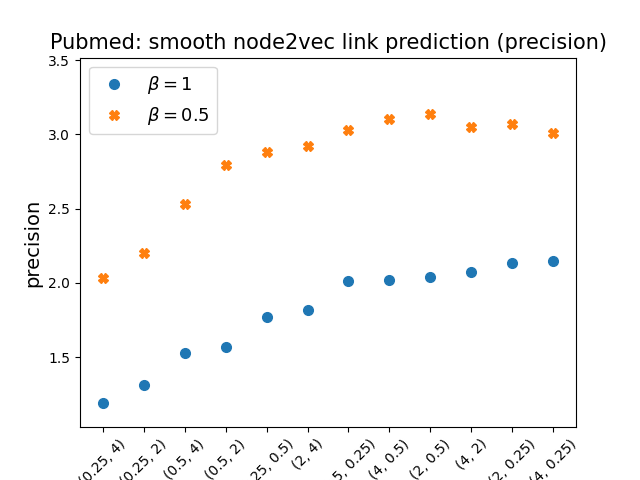} 
\includegraphics[width=60mm]{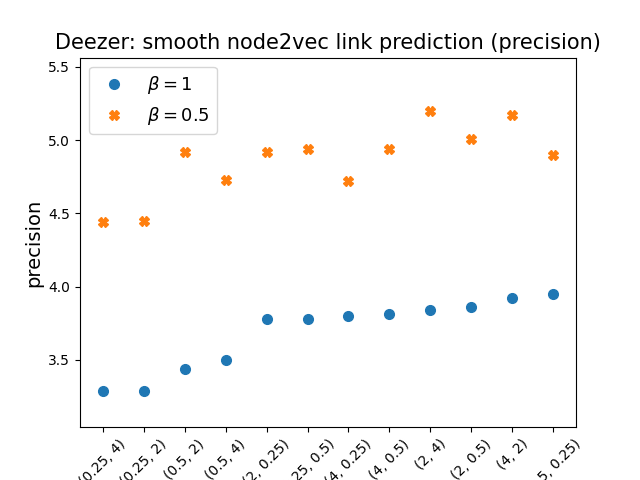}
\\ 
\includegraphics[width=60mm]{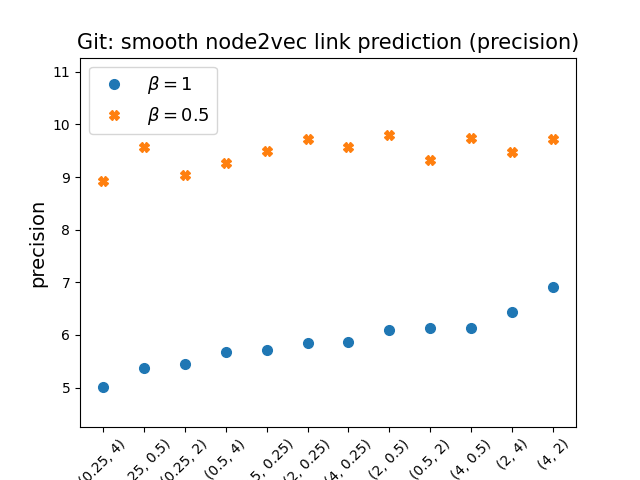} 
\includegraphics[width=60mm]{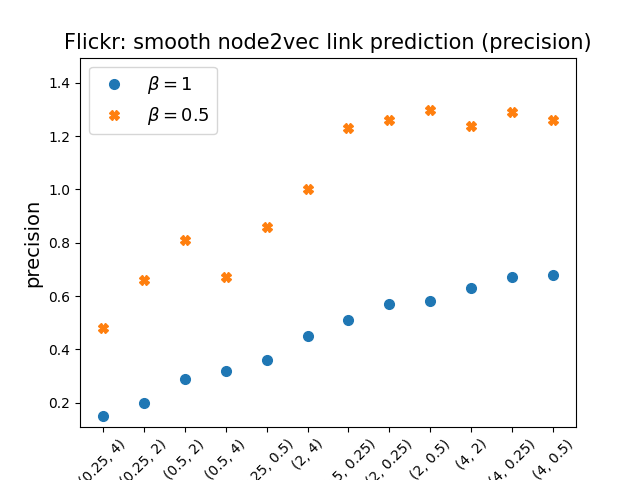}
\\ 
\includegraphics[width=60mm]{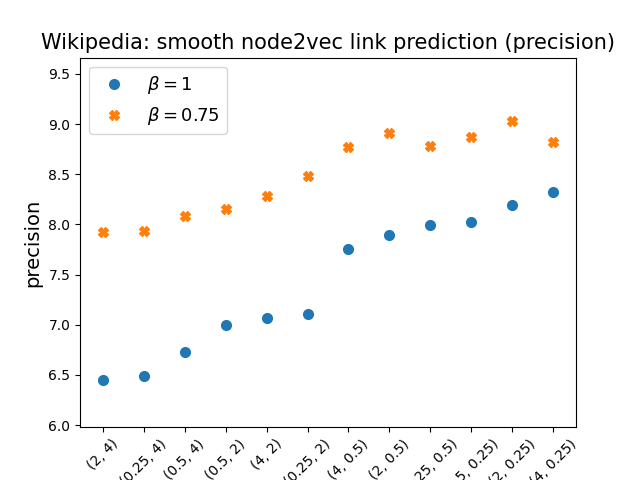} 
\includegraphics[width=60mm]{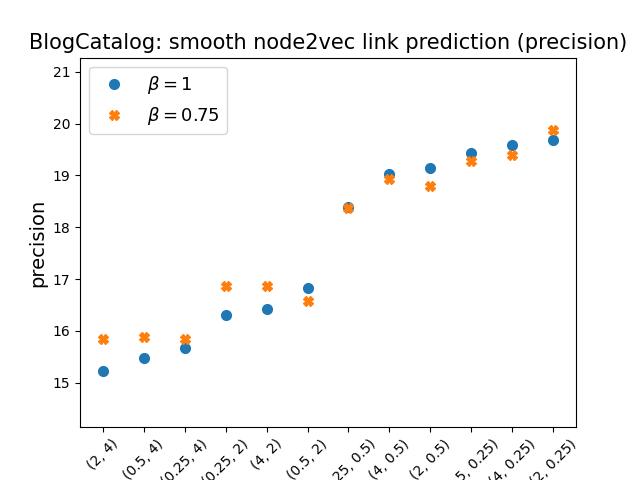} 
\caption{Precision@100 scores for SmoothNode2Vec($p, q$) for varying $p$ and $q$.}
\label{fig:app_n2v_linkpred}
\end{figure*}

\begin{figure*}[h!]
\centering
\includegraphics[width=60mm]{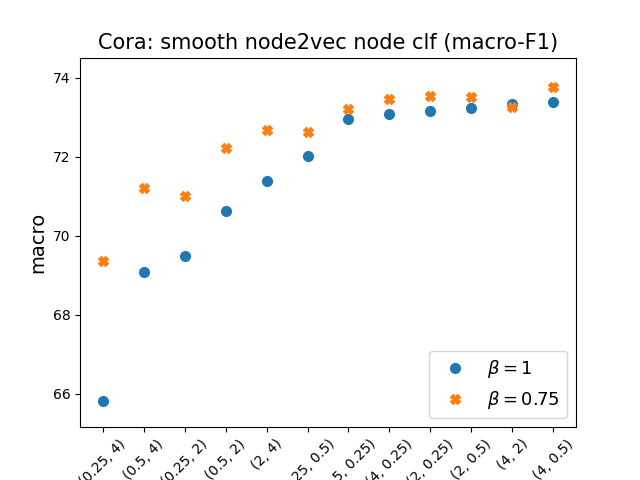} 
\includegraphics[width=60mm]{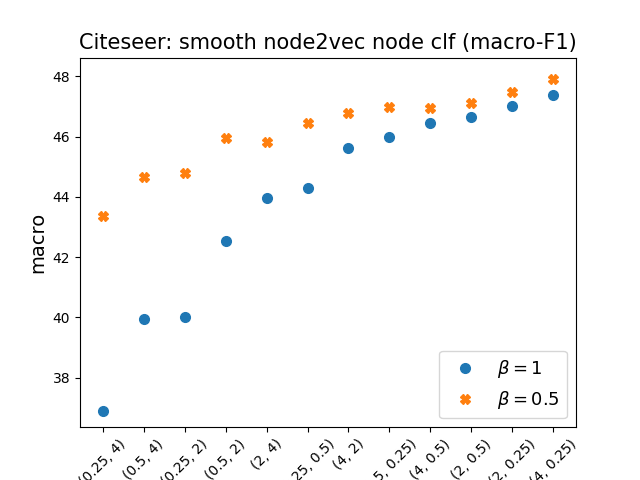} 
\\
\includegraphics[width=60mm]{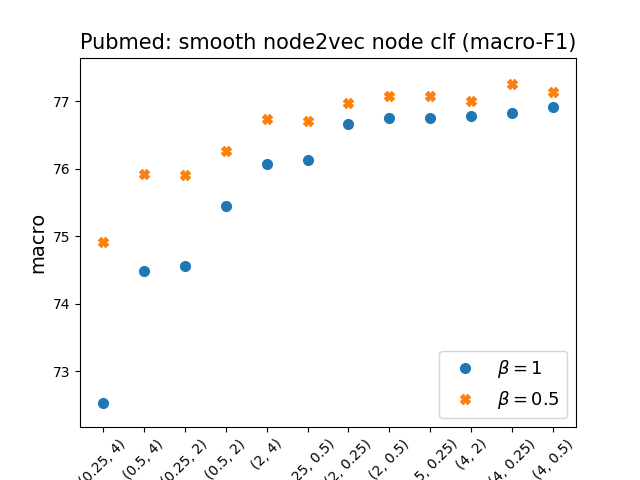} 
\includegraphics[width=60mm]{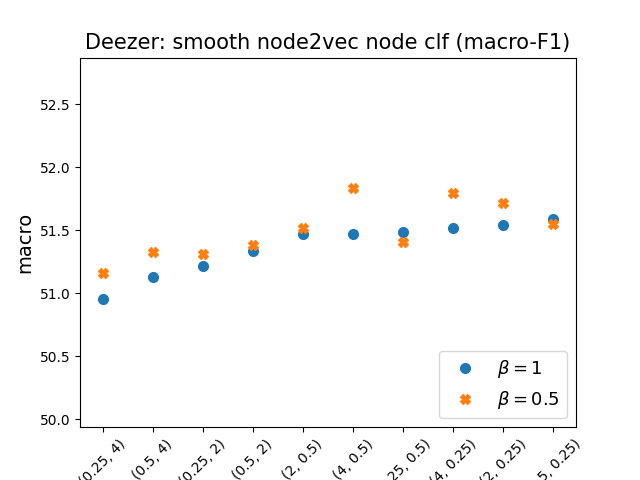}
\\ 
\includegraphics[width=60mm]{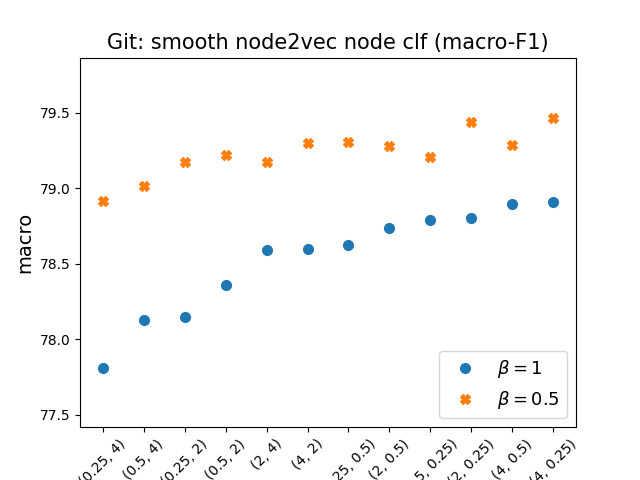} 
\includegraphics[width=60mm]{linkpred_Flickr_precision.png}
\\ 
\includegraphics[width=60mm]{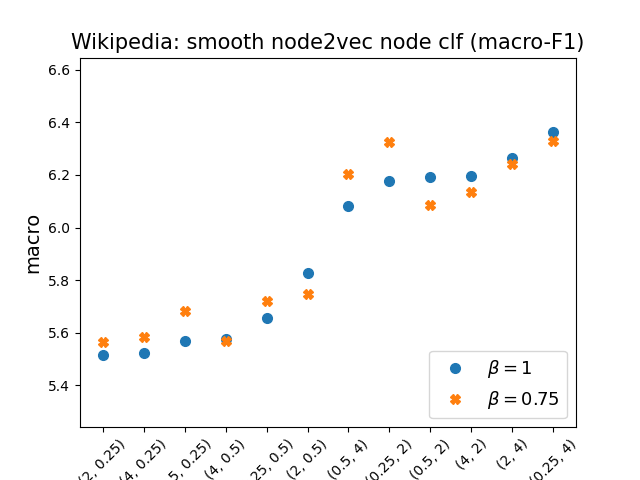} 
\includegraphics[width=60mm]{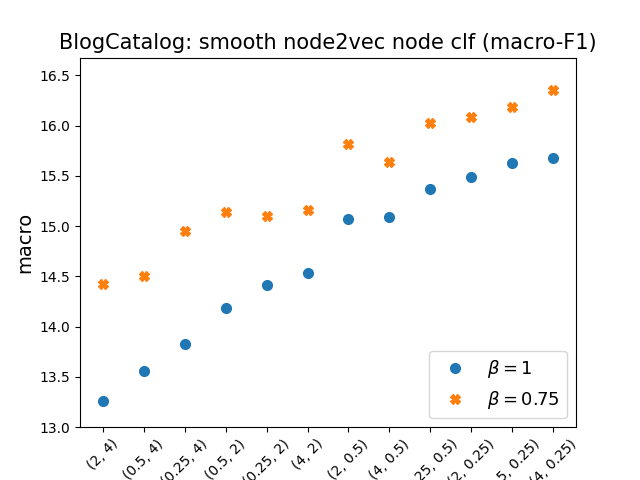} 
\caption{Macro-F1 scores for SmoothNode2Vec($p, q$) for varying $p$ and $q$.}
\label{fig:app_n2v_nodeclf}
\end{figure*}

\subsection{Smooth GraphSage} \label{sec:app_graphsage}

We implemented the unsupervised version of the GraphSage algorithm~\cite{graphsage} and evaluate the effect of pair smoothing for node classification on the four smaller graphs. 

We use for GraphSage the same pair generator we use for SmoothDeepWalk and performed experiments in Google Colab notebook with a T4 GPU with 16 GB of memory. We consider GraphSage of depth 2, i.e., the 2-hop neighborhood of each node. We sample 15 1-hop neighbors and 10 2-hop neighbors from the local neighborhood and use mean aggregation for the neighboring nodes. As initial node features we use 32-dimensional vectors sampled from the 128-dimensional node embeddings learned by the standard DeepWalk algorithm. The learnable weight matrices have dimensions $W_1 \in \mathbb{R}^{64 \times 32}, W_2 \in \mathbb{R}^{128 \times 64}$, and we use LeakyReLU with $\alpha=0.3$ as a non-linear activation function.

 As we see in Table~\ref{tab:app_graphsage}, smoothing is to some degree beneficial in a deep learning setting but the gains are more modest. 
\begin{table*}[t!]
\centering
\caption{Node classification and link prediction scores for GraphSage and its smoothed version with the recommended values for $b$ and $\beta$. The gray color indicates the  difference is not statistically significant.}
\label{tab:app_graphsage}
\begin{tabular}{ll cccc}
 & & Cora & Citeseer & Pubmed & Wikipedia\\ 
\toprule
\multirow{3}{*}{Macro-F1}
& GraphSage  & 73.74 & 45.22 & 76.51 & 4.24\\
& smoothGraphSage & 74.66 & 46.37 & 76.92 & 4.85\\
\midrule
\multirow{3}{*}{Micro-F1}
& GraphSage  & 76.01 & 51.7 & 78.24 & 18.04\\
& smoothGraphSage & 76.61 & 52.73 & {\color{gray} 78.48} & 19.87\\
\bottomrule
\toprule
\multirow{3}{*}{Precision@100}
& GraphSage  & 1.62 & 0.83 & 2.02 & 7.92 \\
& smoothGraphSage & {\color{gray} 1.65} & {\color{gray}  0.83} & 2.17 & {\color{gray} 8.06}\\
\midrule
\multirow{3}{*}{Recall@100}
& GraphSage  & 73.31 & 53.67 & 11.96 & 21.51 \\
& smoothGraphSage & {\color{gray} 73.14} & {\color{gray} 54.33} & 12.68 & {\color{gray} 22.0}\\
\bottomrule
\end{tabular}
\end{table*}

\subsection{Comparison with other embedding methods} \label{sec:app_competitors}


In Table~\ref{tab:app_comp_clf} and Table~\ref{tab:app_comp_lp} we compare the performance of SmoothDeepWalk (SDW) to other approaches for node classification and link prediction, respectively. For a fair comparison, we use the default values for $\beta=0.5$ or $\beta=0.75$ and $b=10\%$. 
The aggregated scores reported in Table~\ref{tab:aggregate} in the main paper are based on the scores in these two tables.

A special attention deserves our direct competitor, the smooth negative sampling algorithm by Yang et al.~\cite{neg_smooth}, denoted as SNS. We copied the sampling function into our framework and set the negative smoothing parameter to 0.5 which corresponds to $\beta=0.5$ for SDW. The results are clearly worse than ours and the SNS sampling approach is very slow. SNS would need more than 32 hours for Git, and more than 48 hours for Flickr, a time increase of more than 30 times compared to SDW, see Table~\ref{tab:app_neg_smooth}. The adaptive negative sampling approach (NEPS)~\cite{adaptive_neg_sampling} penalizes sampling high degree nodes and can be seen a heuristic that achieves frequency smoothing. The running time increase is even larger as the approach needs time $O(n)$ to generate a positive sample. 

\begin{table*}[t]
\caption{Time increase factor for SmoothDeepWalk (SDW) with $\beta=0.5, b=10\%$ and Smooth Negative Sampling (SNS) with $\alpha=0.5$. The first row shows the running time for DeepWalk in seconds, and the next two rows show the time increase for SDW and SNS.}
\centering
\begin{tabular}{l cccccccc}
& Cora & Citeseer & Pubmed & Git & Deezer & Flickr & Wikipedia & BlogCatalog\\
\toprule
DW  & 335  & 456 & 2,330 & 3,606 & 5,288 & 24,125 & 480 & 1030 \\
SDW  & 2.8x & 2.9x & 2.1x & 1.7x & 1.6x & 1.2x & 1.7x & 1.9x\\
SNS & 66x & 65x & 38x & 40x & 42x & 41x & 57x & 48x\\
NEPS & 89x & 95x & 69x & 73x & 79x & 101x & 85x & 91x\\
\bottomrule
\end{tabular}
\label{tab:app_neg_smooth}
\end{table*}

We consider several representative time efficient embedding methods such as LINE~\cite{line},  NetMF~\cite{matfac}, Walklets~\cite{walklets}, Verse~\cite{verse}, and res2vec~\cite{res2vec}. We kept all common hyperparameters identical (number of samples, embedding dimension, walk length, etc.) and use otherwise the default values for the hyperparameters of the corresponding methods. We provide details about the implementation of the different algorithms in the next paragraph. 

Confirming the observations in~\cite{survey_embs}, the behavior of DeepWalk is very stable across different datasets and problems. SmoothDeepWalk consistently improves upon DeepWalk and yields results close to the best ones. In contrast, LINE, Verse, Walklets and NetMF yield the best results in some cases but clearly fail in others.   For link prediction SmoothDeepWalk is clearly the best approach. Even if SmoothDeepWalk doesn't yield the best result for any of the graphs for node classification, it is always close to the best result. 

When aggregating the results from Table~\ref{tab:app_comp_clf} and Table~\ref{tab:app_comp_lp} into Table~\ref{tab:aggregate} and Table~\ref{tab:app_agg}, we take for LINE the better result from LINE1 and LINE2, for NetMF -- from NetMF1 and NetMF2, and NetMF and for res2vec -- from res2vec(MF) and res2vec(SGD). Details about the different methods can be found in the next paragraph.


\begin{table*}[t!]
\centering
\caption{Precision@100 and Recall@100 scores for link prediction for SmoothDeepWalk (SDW), SmoothNode2Vec (SN2V) and other methods. OOM stands for out-of-memory. The best score is shown in {\bf bold}, more than one score is in bold if there is no statistically significant difference between the best scores.}
\label{tab:app_comp_lp}
\begin{tabular}{ll  cccccccc}
& & Cora & Citeseer & Pubmed & Deezer & Git & Flickr & Wiki & Blog\\
\toprule
& & \multicolumn{8}{c}{\large \texttt{Precision@100 scores}}\\
\toprule
	& DW & {\bf 1.55} & {\bf 0.86} & 1.92 & 3.9 & 6.36 & 0.49 & 7.88 & 18.11\\
	\midrule
	 & N2V & {\bf 1.54} & {\bf 0.89} & 2.15 & 3.8 & 5.87 & 0.67 & 8.32 & {\bf 19.58}\\
	\midrule
	 & SNS & {\bf 1.54} & {\bf 0.85} & $>24$h & $>24$h & $>24$h & $>24$h & 6.82 & $>24$h\\
	\midrule
	& NEPS & {\bf 1.56} & {\bf 0.87} & $>24$h & $>24$h & $>24$h & $>24$h & 8.12 & $>24$h\\
	\midrule 
	 & Walklets & 1.34 & 0.28 & 0.55 & 1.95 & 0.31 & 0.0 & 0.39 & 0.07\\
	\midrule
	 & Verse & 1.33 & 0.46 & 2.03 & 3.48 & 2.91 & 0.58 & 2.86 & 1.52\\
	\midrule
	 & NetMF1 & {\bf 1.5} & 0.42 & 1.03 & 1.68 & 0.02 & 0.03 & 2.13 & 1.9\\
	\midrule
	 & NetMF2 & {\bf 1.52} & 0.43 & 1.23 & 2.43 & 0.02 & 0.8 & 0.7 & 2.48\\
	\midrule
	 & Line1 & {\bf 1.46} & 0.32 & 0.62 & 2.86 & 0.82 & 1.34 & 0.94 & 1.42\\
	\midrule
	 & Line2 & 1.28 & 0.38 & 0.2 & 1.68 & 0.02 & 0.0 & 0.49 & 0.01\\
	\midrule
	 & res2vec(MF) & 1.44 & 0.38 & 0.0 & 0.8 & 0.02 & OOM & 1.27 & 1.04\\
	\midrule
	 & res2vec(SGD) & 1.42 & 0.66 & 0.9 & 1.29 & 0.72 & {\bf 2.4} & 0.29 & 0.34\\
	 \midrule
	 \midrule
	 & SDW & 1.53 & 0.86 & {\bf 3.01} & {\bf 4.94} & {\bf 9.54} & 1.13 & {\bf 9.11} & 17.67\\
	 \midrule
	 & SN2V & {\bf 1.54} & {\bf 0.89} & {\bf 3.03} & {\bf 4.72} & {\bf 9.58} &  1.29 & 8.82 & {\bf 19.39}\\
\bottomrule
\\
\toprule
& & \multicolumn{8}{c}{\large \texttt{Recall@100 scores}}\\
& & Cora & Citeseer & Pubmed & Deezer & Git & Flickr & Wiki & Blog\\
\toprule
	 & DW & {\bf 75.43} & 60.17 & 11.46 & 10.85 & 5.51 & 0.27 & 21.97 & 13.66\\
	\midrule
	& N2V & {\bf 75.2} & {\bf 62.5} & 12.96 & 10.59 & 5.08 & 0.37 & 22.86 & {\bf 14.78}\\
	\midrule
	 & SNS & {\bf 75.38} & 60.15 & $>24$h & $>24$h & $>24$h & $>24$h & 18.82 & $>24$h\\
	\midrule
	& NEPS & {\bf 75.47} & {\bf 60.6} & $>24$h & $>24$h & $>24$h & $>24$h & 21.02 & $>24$h\\
	\midrule 
	 & Walklets & 57.2 & 19.5 & 3.3 & 5.3 & 0.27 & 0.0 & 1.08 & 0.05\\
	\midrule
	 & Verse & 59.78 & 29.42 & 12.21 & 9.42 & 2.52 & 0.32 & 7.73 & 1.13\\
	\midrule
	 & NetMF1 & 66.19 & 26.25 & 6.09 & 4.56 & 0.01 & 0.02 & 5.69 & 1.41\\
	\midrule
	 & NetMF2 & 64.89 & 27.08 & 7.39 & 6.62 & 0.01 & 0.44 & 1.89 & 1.86\\
	\midrule
	 & Line1 & 61.27 & 21.17 & 3.76 & 7.84 & 0.7 & 0.75 & 2.58 & 1.06\\
	\midrule
	 & Line2 & 56.7 & 25.67 & 1.25 & 4.57 & 0.02 & 0.0 & 1.33 & 0.01\\
	\midrule
	 & res2vec(MF) & 61.7 & 24.0 & 0.0 & 2.19 & 0.02 & OOM & 3.4 & 0.77\\
	\midrule
	 & res2vec(SGD) & 60.7 & 41.17 & 5.22 & 3.51 & 0.63 & {\bf 1.59} & 0.79 & 0.25\\
	 \midrule
	 \midrule
	 & SDW & {\bf 74.98} & 60.67 & {\bf 17.81} & {\bf 13.76} & {\bf 8.27} & 0.63 & {\bf 25.08} & 13.34\\
	\midrule
	 & SN2V & {\bf 75.6} & {\bf 62.5} & {\bf 18.03} &  13.12 & {\bf 8.3} & 0.72 & 24.23 & {\bf 14.63}\\
\bottomrule
\end{tabular}
\end{table*}

\begin{table*}[t!]
\centering
\caption{Macro-F1 and Micro-F1 scores for node classification for SmoothDeepWalk (SDW), SmoothNode2Vec (SN2V) and other methods. OOM stands for out-of-memory. The best score is shown in {\bf bold}, more than one score is in bold if there is no statistically significant difference between the best scores.}
\label{tab:app_comp_clf}
\begin{tabular}{ll  cccccccc}
& & Cora & Citeseer & Pubmed & Deezer & Git & Flickr & Wiki & Blog\\
\toprule
& & \multicolumn{8}{c}{\large \texttt{Macro-F1 scores}}\\
\toprule
	& DW & 72.91 & 44.29 & 73.34 & 51.63 & 78.82 & 24.37 & 6.02 & 14.94\\
	\midrule
	 & N2V & 73.09 & 47.37 & 76.83 & 51.47 & 78.91 & 24.22 & 5.52 & 15.68\\
	\midrule
	 & SNS & 72.59 & 44.11 & $>24$h & $>24$h & $>24$h & $>24$h & 5.15 & $>24$h\\
	\midrule
	& NEPS & {\bf 73.79} &  45.17 & $>24$h & $>24$h & $>24$h & $>24$h & 5.32 & $>24$h\\
	\midrule 
	 & Walklets & 70.31 & 44.84 & 76.63 & {\bf 52.19} & {\bf 79.96} & 24.5 & 5.81 & 12.07\\
	\midrule
	 & Verse & 53.87 & 31.18 & 54.83 & 50.18 & 76.84 & {\bf 25.3} & 4.8 & 14.66\\
	\midrule
	 & NetMF1 & 66.99 & 43.84 & 73.16 & 48.43 & 74.07 & 20.31 & 4.97 & {\bf 20.15}\\
	\midrule
	 & NetMF2 & 69.09 & 45.25 & {\bf 77.32} & 49.24 & 75.98 & 22.05 & 3.91 & 18.12\\
	\midrule
	 & Line1 & 60.47 & 40.29 & 74.25 & 51.71 & 77.41 & 23.32 & 0.88 & 14.26\\
	\midrule
	 & Line2 & 59.05 & 34.04 & 73.16 & 51.33 & 79.0 & 23.34 & {\bf 7.52} & 14.11\\
	\midrule
	 & res2vec(MF) & {\bf 74.02} & 45.06 & 75.96 & 51.69 & 76.51 & OOM & 1.95 & 7.3\\
	\midrule
	 & res2vec(SGD) & 72.56 & 47.47 & 72.63 & 50.06 & 71.26 & 15.88 & 2.22 & 0.53\\
	 \midrule
	 \midrule
	 & SDW & 73.23 & 46.66 & 76.76 & 51.58 & 79.3 & 24.78 & 5.94 & 15.74\\
	 \midrule
	 & SN2V & 73.46 & {\bf 47.9} & {\bf 77.27} & 51.52 & 79.47 & 24.82 & 5.58 & 16.35\\
\bottomrule
\\
\toprule
& & \multicolumn{8}{c}{\large \texttt{Micro-F1 scores}}\\
& & Cora & Citeseer & Pubmed & Deezer & Git & Flickr & Wiki & Blog\\
\toprule
	 & DW & 74.38 & 48.2 & 75.02 & 54.22 & 84.82 & 51.7 & 30.86 & 23.93\\
	\midrule
	 & N2V & 74.48 & 51.05 & 78.3 & 54.07 & 84.87 & 51.44 & 29.74 & 24.54\\
	\midrule
	& SNS & 74.05 & 48.01 & 75.6 & $>24$h & $>24$h & $>24$h & 29.3 & $>24$h \\
	\midrule
	& NEPS &  74.8 &  50.25 & $>24$h & $>24$h & $>24$h & $>24$h & 32.67 & $>24$h\\
	\midrule 
	 & Walklets & 71.94 & 48.09 & 78.15 & 53.9 & {\bf 85.54} & {\bf 52.6} &  37.01 & 25.2\\
	\midrule
	 & Verse & 56.89 & 33.96 & 58.62 & 53.34 & 83.75 & 52.49 & 32.98 & 27.61\\
	\midrule
	 & NetMF1 & 68.51 & 48.0 & 76.02 & {\bf 55.29} & 82.94 & 50.9 & 30.57 & {\bf 33.52}\\
	\midrule
	 & NetMF2 & 70.5 & 49.23 & {\bf 78.91} & 55.02 & 83.86 & 51.54 & 27.1 & 29.03\\
	\midrule
	 & Line1 & 62.31 & 44.03 & 75.96 & 54.65 & 84.42 & 51.25 & 18.55 & 30.18\\
	\midrule
	 & Line2 & 60.91 & 37.27 & 74.85 & 54.34 & 84.83 & 51.96 & {\bf 39.4} & 25.57\\
	\midrule
	 & res2vec(MF) & {\bf 75.47} & 49.59 & 77.63 & 54.66 & 84.24 & OOM & 30.13 & 17.5\\
	\midrule
	 & res2vec(SGD) & 74.24 & {\bf 51.68} & 74.42 & 53.1 & 80.69 & 44.71 & 26.51 & 1.33\\
	 \midrule
	 \midrule
	  & SDW & 74.58 & 50.14 & 78.27 & 54.19 & 85.14 & 52.18 & 30.38 & 24.68\\
	  \midrule
	 & SN2V & 74.85 & 51.14 & {\bf 78.71} & 54.36 & {\bf 85.32} & 52.0 & 29.23 & 25.24\\
	\midrule
\bottomrule
\end{tabular}
\end{table*}

\begin{table*}
\caption{Aggregated recall@100 scores for link prediction (left) and micro-F1 scores for node classification (right) for SmoothDeepWalk, SmoothNode2Vec and other methods.}
\label{tab:app_agg}
\centering
\begin{minipage}{.5\linewidth}
    \centering
    \begin{tabular}{ lcc }
      \toprule
      method & mean  &  min \\
      \toprule
      	 SN2V & 0.924 & 0.538 \\
      	\midrule
	 SDW & 0.897 & 0.471 \\
	\midrule
	 N2V & 0.779 & 0.279 \\
	\midrule
	 DW & 0.75 & 0.204 \\
	 \midrule
	 Verse & 0.455 & 0.078 \\
	 \midrule
	 res2vec & 0.43 & 0.053 \\
	\midrule
	 NetMF & 0.376 & 0.002 \\
	 \midrule
	 Line & 0.365 & 0.073 \\
	\midrule
	 Walklets & 0.223 & 0.01 \\
      \bottomrule
    \end{tabular}
  \end{minipage}%
  \begin{minipage}{.5\linewidth}
    \centering
      \begin{tabular}{ lcc }
      \toprule
       method & mean  &  min\\
       \toprule
       NetMF & 0.949 & 0.776 \\
	\midrule
	Walklets & 0.938 & 0.752 \\
	\midrule
	Line & 0.935 & 0.813 \\
	\midrule
       SN2V & 0.926 & 0.742 \\
	\midrule
	N2V & 0.921 & 0.732 \\
	\midrule
	SDW & 0.92 & 0.725 \\
	\midrule
	DW & 0.911 & 0.714 \\
	\midrule
	Verse & 0.841 & 0.644 \\
	\midrule
	res2vec & 0.882 & 0.522 \\
		\bottomrule
    \end{tabular}
  \end{minipage}
\end{table*}

\paragraph{Implementation details}

For the other embedding approaches evaluated in the paper we use the following:
\begin{itemize}
\item 
We copied the implementation of {\bf smooth negative sampling (SNS)}~\cite{neg_smooth} provided by the authors~\url{https://github.com/THUDM/MCNS} into our framework. The computation of several inner products between embedding vectors results in an extremely slow processing of the corpus. Even after running the code on a more powerful machine with a 32 core CPU and an NVIDIA 4080 GPU, training on the larger graphs is still prohibitively slow. (In their evaluation, Yang et al. use a server with 8 state-of-the-art GPU powered machines.)  Hence, we only present results for the three smaller graphs Cora, Citeseer and Wikipedia. The results in Table~\ref{tab:app_neg_smooth} show the increase in running time compared to the standard DeepWalk algorithm. (For the larger graphs this is the estimated running time returned by TensorFlow.) As we see, the running time increase for SmoothDeepWalk for the larger Git and Flickr is rather modest while it is orders of magnitude for the SNS algorithm. Note that the gap would increase for larger embedding sizes $d$ as the sampling procedure in~\cite{neg_smooth} depends on $d$. 
\item node2vec. We used the implementation available at \url{https://github.com/aditya-grover/node2vec} to generate a corpus and then apply smooth pair sampling to it.
\item NEPS. We use the public implementation provided by the authors: \url{https://github.com/Andrewsama/NEPS-master}.
\item Verse. The algorithm uses random walk with restarts with probability $1-\alpha$, similarly to PageRank. We use the implementation provided by the authors \url{https://github.com/xgfs/verse}, using for probability for walk restart the default value of $\alpha=0.85$. 
\item LINE. The method learns to preserve the first and second proximities of each node, hence the versions Line1 and Line2. We use the code \url{https://github.com/tangjianpku/LINE}. For the total number of sampled node pairs we use the same number as the one generated from the random walk corpus in (Smooth)DeepWalk.
\item Walklets works by subsampling random walks. We use the public implementation provided By Karate Club~\cite{karateclub} \url{https://karateclub.readthedocs.io/en/latest/index.html}
\item NetMF factorizes (a sparse version of) the node similarity matrix discussed in Section~\ref{sec:why}. We also use the Karate Club implementation. We consider order 1 and 2 corresponding to different powers of the node similarity matrix.  (It is worth noting that for larger powers we obtain out-of-memory errors for the larger graphs.) 
\item   Residual2Vec corrects biases introduced by random walks by designing novel techniques for the generation of negative samples that reflect the graph structure. We use the public implementation from \url{https://github.com/skojaku/residual2vec}. We use the default setting, in particular we do not specify group membership for the nodes. We run experiments with two variants of the approach based on matrix factorization, denoted as res2vec(MF), and on stochastic gradient descent -- res2vec(SGD). For res2vec(SGD) we use the the stochastic block model method for negative sampling. 
\end{itemize}

\end{document}